\documentclass{article}

\usepackage{microtype}
\usepackage{graphicx}
\usepackage{subfigure}
\usepackage{booktabs}
\usepackage{multirow}

\usepackage{hyperref}

\usepackage[accepted]{icml2024}

\usepackage{amsmath}
\usepackage{amssymb}
\usepackage{mathtools}
\usepackage{amsthm}

\usepackage[capitalize,noabbrev]{cleveref}
\crefname{assumption}{Assumption}{Assumptions}
\crefname{equation}{Eq.}{Eqs.}

\usepackage{enumitem}
\usepackage{pifont}

\newcommand{\cmark}{\ding{51}}
\newcommand{\xmark}{\ding{55}}
\DeclareMathOperator{\tr}{tr}
\DeclareMathOperator{\sg}{\operatorname{sg}}

\usepackage{thmtools}
\usepackage{thm-restate}

\usepackage{xspace}
\makeatletter
\DeclareRobustCommand\onedot{\futurelet\@let@token\@onedot}
\def\@onedot{\ifx\@let@token.\else.\null\fi\xspace}

\def\eg{\emph{e.g}\onedot} 
\def\ie{\emph{i.e}\onedot}

\makeatother

\theoremstyle{plain}
\newtheorem{theorem}{Theorem}[section]

\theoremstyle{definition}

\newtheorem{assumption}[theorem]{Assumption}
\theoremstyle{remark}

\usepackage[textsize=tiny]{todonotes}

\icmltitlerunning{On the Effectiveness of Supervision in Asymmetric Non-Contrastive Learning}

\begin{document}

\twocolumn[
\icmltitle{On the Effectiveness of Supervision in Asymmetric Non-Contrastive Learning}

\icmlsetsymbol{equal}{*}

\begin{icmlauthorlist}
\icmlauthor{Jeongheon Oh}{yonsei}
\icmlauthor{Kibok Lee}{yonsei}
\end{icmlauthorlist}

\icmlaffiliation{yonsei}{Department of Statistics and Data Science, Yonsei University}

\icmlcorrespondingauthor{Kibok Lee}{kibok@yonsei.ac.kr}

\icmlkeywords{Supervised Representation Learning, Supervised Non-Contrastive Learning}

\vskip 0.3in
]

\printAffiliationsAndNotice{}

\begin{abstract}
Supervised contrastive representation learning has been shown to be effective in various transfer learning scenarios.
However, while asymmetric non-contrastive learning (ANCL) often outperforms its contrastive learning counterpart in self-supervised representation learning, the extension of ANCL to supervised scenarios is less explored.
To bridge the gap, we study ANCL for supervised representation learning, coined \textsc{SupSiam} and \textsc{SupBYOL}, leveraging labels in ANCL to achieve better representations.
The proposed supervised ANCL framework improves representation learning while avoiding collapse.
Our analysis reveals that providing supervision to ANCL reduces intra-class variance, and the contribution of supervision should be adjusted to achieve the best performance.
Experiments demonstrate the superiority of supervised ANCL across various datasets and tasks.
The code is available at: \url{https://github.com/JH-Oh-23/Sup-ANCL}.
\vspace{-10pt}
\end{abstract}

\begin{figure*}[t]
\begin{center}
\includegraphics[width=.9\textwidth]{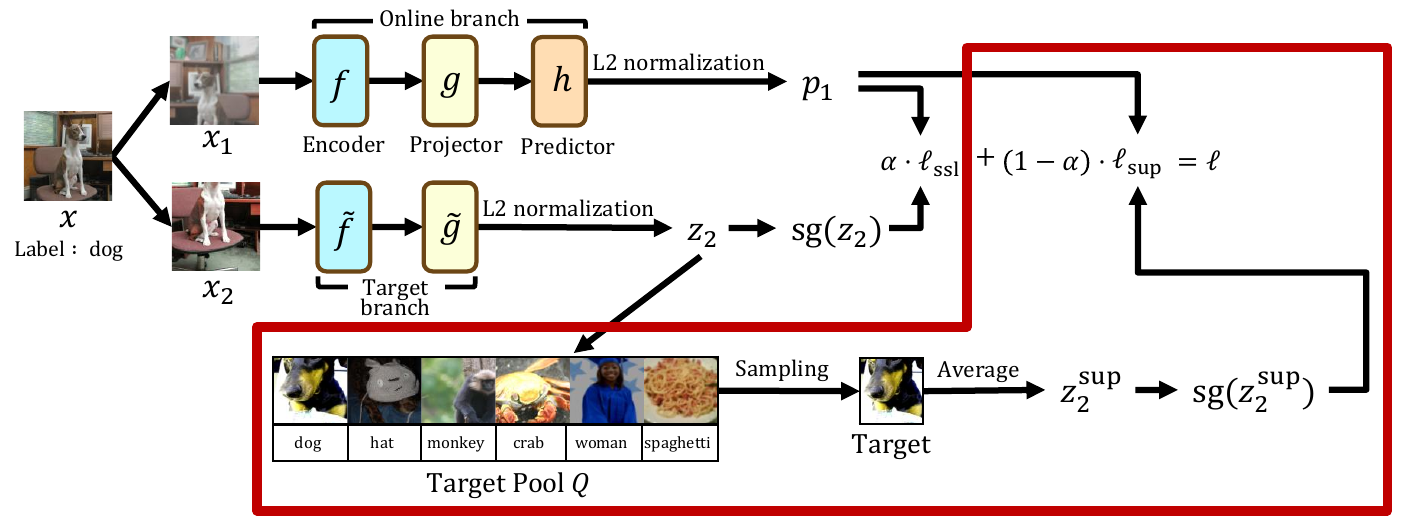}
\vspace{-5pt}
\caption{\textbf{Our proposed supervised ANCL framework}.
The components we added to the standard ANCL are highlighted with a red box.
We manage a target pool to ensure the existence of positive samples sharing the same class label in the form of $z_2^\text{sup}$.
Stop-gradient ($\sg$) applied to $z_2$ and $z_2^\text{sup}$
ensures that the gradients propagate through the online branch with the predictor only.
The target branch without the predictor either shares parameters with the online branch (\textsc{SupSiam}), or exhibits a momentum network (\textsc{SupBYOL}).
} 
\label{fig:framework}
\end{center}
\vspace{-5pt}
\end{figure*}

\section{Introduction}

Self-supervised learning has recently been proven to be an effective paradigm for representation learning~\citep{chen20simple,chen21siamese,he2022masked}.
Among various pretext tasks for self-supervised learning, contrastive learning (CL)~\citep{oord18cpc,chen20simple,he20contrast} first promised outstanding performance, surpassing the transfer learning performance of supervised pretraining~\citep{sharif2014cnn}, which learns representations by attracting positive pairs while repelling negative pairs.
However, CL requires negative samples to ensure good performance, which might not be possible under limited batch sizes.
On the other hand, asymmetric non-contrastive learning (ANCL)~\citep{grill20bootstrap,chen21siamese} has emerged as a promising alternative to CL, which maximizes the similarity between positive pairs without relying on negative samples.
To prevent learned representations from collapsing, ANCL employs an asymmetric structure by placing a predictor after one side of the projector.

A key component in both CL and ANCL is acquisition of positive pairs, which is typically achieved through data augmentation.
Given that datasets for pretraining often include labels, \citet{khosla20supcon} proposed to incorporate supervision into CL by treating samples with the same class label as positive pairs as well.
Supervised CL has demonstrated superior performance across diverse tasks, such as few-shot learning \citep{majumder21supmoco}, long-tail recognition \citep{kang21balance}, continual learning \citep{cha2021co2l}, and natural language processing \citep{gunel21}.

While supervision helps to discover more positive samples, it does not directly help to identify effective negative samples.
Consequently, ANCL has a better potential to benefit from supervision, as it focuses on positive pairs only.
However, in contrast to CL, there are a limited number of studies on leveraging supervision to improve ANCL, despite its strong performance in self-supervised learning.
To bridge this gap, we study the effect of supervision in ANCL by introducing the supervised ANCL framework and investigating its behavior through theoretical and empirical analysis.

To the best of our knowledge, our work is the first to conduct a theoretical analysis of the behavior of representations learned through supervised ANCL.
Our experiments confirm the effectiveness of supervision observed through our theoretical analysis, as well as the superiority of representations learned via supervised ANCL across various datasets and tasks.
Specifically, as illustrated in \cref{fig:framework}, we consider \textsc{SupSiam} and \textsc{SupBYOL}, which are supervised adaptations of the two popular ANCL methods, \textsc{SimSiam} \citep{chen21siamese} and \textsc{BYOL} \citep{grill20bootstrap}, respectively.
Our contributions are summarized as follows:
\begin{itemize}[leftmargin=12pt]

    \item We propose a supervised ANCL framework for representation learning while avoiding collapse, which surpasses the performance of its self-supervised counterpart when supervision is available.

    \item Our analysis demonstrates that incorporating supervision into ANCL reduces the intra-class variance of latent features, and that learning to capture both intra- and inter-class variance is crucial for representation learning.

    \item Our experiments validate our analysis and demonstrate the superiority of representations learned via supervised ANCL across various datasets and tasks.

\end{itemize}

\section{Related Works}

\textbf{Supervised CL}.
Although \textsc{SupCon} \citep {khosla20supcon} demonstrated remarkable linear probing performance on pretrained datasets, its efficacy on other downstream datasets is comparable or inferior to that of self-supervised methods. 
In response, subsequent works have been underway to better utilize supervision to enhance representation learning.
\citet{wei21semantic} proposed to improve CL by taking top-$K$ positive neighbors into account and assigning soft labels to positive samples based on similarity, such that it better reflects task-specific semantic features and task-agnostic appearance features.
\citet{wang23opera} argued that naively incorporating supervised signals might conflict with the self-supervised signals.
To address this issue, \citet{wang23opera} proposed to impose hierarchical supervisions with an additional projector.
\citet{graf21dissect} provided both theoretical and empirical evidence demonstrating that the \textsc{SupCon} loss is minimized when each class collapses to a single point, resulting in poor generalization of learned representations.
\citet{chen22perfectly} found that the \textsc{SupCon} loss is invariant to class-fixing permutations, indicating that the loss remains unchanged when data points within the same class are arbitrarily permuted in representation space, which also leads to poor generalization of learned representations.
\citet{chen22perfectly} proposed incorporating a weighted class-conditional InfoNCE loss to avoid class collapse, and constraining the encoder, adding a class-conditional autoencoder, and using data augmentation to break permutation invariance.
\citet{xue23learnt} argued that features learned through supervised CL are prone to class collapse, whereas those learned through self-supervised CL suffer from feature suppression, \ie, easy and class-irrelevant features suppress to learn harder and class-relevant features.
They claimed that balancing the losses of supervised and self-supervised CL is crucial for improving the quality of learned representations.
Notably, these efforts have concentrated on CL, motivating us to investigate the effect of supervision in ANCL.
Although several studies on supervised ANCL exist, such as \citet{asadi22supbyol} and \citet{maser23supsiam}, their contributions lack a theoretical understanding of the effect of supervision and/or are limited to specific domains.

\textbf{Theoretical Analysis on ANCL}.
While the initial ANCL works \citep{grill20bootstrap,chen21siamese} have demonstrated impressive performance, the learning dynamics that enable effective representation learning without negative pairs while avoiding collapse to trivial solutions remain unclear.
\citet{tian21} elucidated the dynamics of ANCL through the spectral decomposition of the correlation matrix.
Specifically, assuming the predictor is linear, they proved that the eigenspace of the learned predictor aligns with the eigenspace of the correlation matrix of the latent features.
\citet{liu22gap} empirically observed that, as learning progresses, both the linear predictor and the correlation matrix of latent features converge to a (scaled) identity matrix in ANCL.
Based on this observation, they argued that the asymmetric architecture in ANCL implicitly encourages feature decorrelation, achieving a similar effect to symmetric non-CL methods that explicitly decorrelate features such as Barlow Twins \citep{zbontar21barlow} and VICReg \citep{bardes22vicreg}.
\citet{zhuo23rdm} suggested that the predictor in ANCL operates as a low-pass filter, thereby decreasing the rank of the predictor outputs.
They argued that the rank difference between the correlation matrix of the projector outputs and that of the predictor outputs mitigates dimensional collapse by gradually increasing the effective rank of them as training progresses.
Inspired by the prior works on self-supervised ANCL, we analyze supervised ANCL under a similar framework with additional assumptions.
On the other hand, \citet{halvagal23implicit} found that prior works overlook the L2 normalization of projector/predictor outputs, which is a common practice in ANCL, before computing the loss.
They investigated the learning dynamics by incorporating the L2 normalization and compared it with the case without the L2 normalization.
Our work also considers the L2 normalization;
however, instead of normalizing the features directly, we consider it as a constraint and employ a Lagrangian formulation.

\section{Method}

In this section, we first review the problem setting of self-supervised ANCL.
Then, we introduce supervised ANCL.
The overall framework is illustrated in \cref{fig:framework}.

\subsection{Preliminary: Self-Supervised ANCL}

Let $f$, $g$, and $h$ be the encoder, projector, and predictor of the online branch, respectively, and $\tilde{f}$ and $\tilde{g}$ be the encoder and projector of the target branch, respectively.
For a data point $x$, let $z = (g \circ f)(x)$ and $p = (h \circ g \circ f)(x)$ be the output of the projector and predictor, respectively.
In self-supervised ANCL, two views $x_1$ and $x_2$ are generated from the data $x$ through augmentation, and the model learns to minimize the distance between these views encoded at different levels:
it compares the prediction of the first view $p_1 = (h \circ g \circ f)(x_1)$ with the projection of the second view $z_2 = (\tilde{g} \circ \tilde{f})(x_2)$.
It has been observed that the asymmetric architecture introduced by the predictor $h$ helps prevent representation collapse by predicting the latent feature of the second view $z_2$ from that of the first view $z_1$, \ie, $z_2 \simeq p_1 = h(z_1)$ \citep{chen21siamese}.
The self-supervised ANCL loss $\ell_{\text{ssl}}$ is expressed as:
\begin{equation} \label{eq:ssl_loss}
    \ell_{\text{ssl}}(p_1,z_2) = \left\|p_1 - \sg\left(z_2\right) \right\|_2^2,
\end{equation}
where $\sg$ is the stop-gradient operation and $p_1$ and $z_2$ are L2-normalized.
The inclusion of stop-gradient is also crucial for preventing collapsing, making it an essential component of the loss formulation \citep{chen21siamese}.

The target branch can either share parameters with the online branch \citep{chen21siamese}, or exhibit a momentum network \citep{grill20bootstrap}.
When a momentum network is employed, its parameters follow the exponential moving average (EMA) update rule:
$\theta_{\tilde{g} \circ \tilde{f}} \leftarrow m \cdot \theta_{\tilde{g} \circ \tilde{f}} + (1-m) \cdot \theta_{g \circ f}$,
where $m$ is the momentum, $\theta_{g \circ f}$ is the set of learnable parameters in $f$ and $g$.
The parameters of the target model $\theta_{\tilde{g} \circ \tilde{f}}$ are initialized to those of the online model $\theta_{g \circ f}$.

\subsection{Supervised ANCL}

We propose to enhance supervised ANCL by integrating supervision through an additional loss function:
for an anchor $x_1$ and its supervised target $x_2^\text{sup}$ sharing the same label $y$, the loss minimizes the distance between $p_1 = (h \circ g \circ f)(x_1)$ and $z_2^\text{sup} = (\tilde{g} \circ \tilde{f})(x_2^\text{sup})$.
However, the additional loss may not always be effective, because the current batch might not contain any samples sharing the same label as the anchor, particularly when the batch size is small.

To address this issue, we introduce a target pool to ensure the presence of targets sharing the same label as each anchor in the batch, regardless of batch size.
Similar to the memory bank utilized in prior works~\citep{wu2018unsupervised}, the target pool $Q$ is a queue storing targets $z_2$ along with their corresponding labels.
The target pool offers another advantage that positive samples from the target pool help mitigate collapse because they are updated more slowly than those sampled from the batch, as empirically observed in \cref{table:ablation_collapse}.
The proposed target pool is flexible in its design, such that it can be a vanilla queue, a collection of per-class queues ensuring the presence of targets from all labels even when the queue size is small, or a set of learnable class prototypes;
the impact of these design choices is investigated in \cref{table:pool-design}.

Now, we sample the supervised target $z_2^\text{sup}$ sharing the same class as the anchor $x$ from the target pool $Q$.
Specifically, we sample $M$ targets and average them to formulate the supervised ANCL loss $\ell_{\text{sup}}$:
{%
\thinmuskip=1mu 
\medmuskip=2mu plus 2mu minus 4mu 
\thickmuskip=3.5mu plus 5mu 
\begin{equation} \label{eq:sup_loss}
    \ell_{\text{sup}}(p_1,z_2^\text{sup}) = \left\| p_1 - \sg\left(z_2^\text{sup}\right)\right\|_2^2,
    \ z_2^\text{sup} = \frac{1}{M}\sum_{z'_2 \in Q_y} z'_2,
\end{equation}
}%
where $p_1$ and $z'_2$ are L2-normalized and $Q_y \subseteq Q$ is the set of $M$ targets sampled from $Q$ sharing the same label $y$ as $x$.
We sample all positives in $Q$ in experiments, and the effect of $M$ is discussed in \cref{sec:numpos}.
Finally, the total loss is defined by the convex combination of $\ell_{\text{ssl}}$ and $\ell_{\text{sup}}$:
\begin{equation} \label{eq:loss}
    \ell(p_1,z_2,z_2^\text{sup}) = \alpha \cdot \ell_{\text{ssl}}(p_1,z_2) + (1-\alpha)\cdot \ell_{\text{sup}}(p_1,z_2^\text{sup}),
\end{equation}
where $\alpha \in [0,1]$ adjusts the contribution of $\ell_{\text{ssl}}$ and $\ell_{\text{sup}}$, and we symmetrize the loss in experiments following the convention.
We argue that the introduction of $\ell_{\text{sup}}$ reduces intra-class variance and $\alpha$ adjusts the amount of reduction, where details can be found in \cref{sec:role}.

Note that our strategy for incorporating supervision into the loss differs from that of \textsc{SupCon} \citep {khosla20supcon}.
We first average the supervised loss before combining it with the self-supervised loss, whereas \textsc{SupCon} weights all per-sample losses equally, regardless of whether they are self-supervised or supervised.
Since our focus is on analyzing the overall effects of self-supervised and supervised losses rather than per-sample losses, our strategy aligns with the analysis presented in the following section.

\section{Analysis of the Effect of Supervision}
\label{sec:theoretical}

In this section, we analyze the effect of supervision in ANCL. 
We argue that incorporating supervision into ANCL reduces intra-class variance, and that its contribution should be adjusted to achieve better representations.
Detailed mathematical proofs are provided in \cref{sec:proof}.

\subsection{Problem Setup}

For simplicity in our analysis, we adopt several assumptions from
\citet{tian21, zhuo23rdm}:
\begin{assumption} \label{assump:linear}
    The encoder followed by the projector $g \circ f$ and the predictor $h$ are linear:
    $z=(g \circ f)(x)=Wx$ and $p = h(z) = W_p z$,
    where $W_p$ is a symmetric matrix.
\end{assumption}

\begin{assumption} \label{assump:aug}
    The distribution of the data augmentation $P(\widetilde{X}|X)$ has a mean $X$ and a covariance matrix $\sigma_e^2 I$.
\end{assumption}
While previous studies on self-supervised ANCL assume that the distribution of the input data has a zero mean and a scaled identity covariance matrix, class-conditional distributions should be considered when incorporating supervision.
Specifically, we assume the class-conditional and class-prior distributions over $C$ classes as follows:
\begin{assumption} \label{assump:class}
    The class-prior distribution follows the uniform distribution:
    $P(Y=y) = 1/C$.
\end{assumption}

\begin{assumption} \label{assump:cond}
    For an input data $X$ and its class $Y$, the conditional distribution $P(X|Y)$ is characterized by a mean $\mu_y$ and a covariance matrix $\Sigma_y$, where the total mean and total covariance matrix are zero and the identity matrix, respectively:
    $\sum_y \mu_y = 0$ and $S_T = \frac{1}{C}\sum_y \left(\mu_y \mu_y^\top +\Sigma_y \right) = I$.
\end{assumption}

\Cref{assump:class} is made for simplicity of analysis;
our analysis holds without this assumption, albeit the derivation becomes more complex.
\Cref{assump:cond} can be naturally satisfied through data whitening.

\subsection{Supervision Reduces Intra-Class Variance}
\label{sec:math}

For simplicity, assume we sample one target from the pool, \ie, $M=1$.
We first express the loss in \cref{eq:loss} with constraints to ensure the L2 normalization of features:
\vspace{-4pt}
\begin{align}
\begin{aligned}[b]
    \label{eq:loss_contrained}
    \ell &= \alpha \left\| W_p z_1 - z_2 \right\|_2^2 + (1-\alpha) \left\| W_p z_1 - z_2^\text{sup} \right\|_2^2 \\
    &= \left\| W_p z_1 - \left( \alpha \cdot z_2 + (1-\alpha) \cdot z_2^\text{sup} \right) \right\|_2^2 + \text{const} \\
    &\text{ s.t. }
    \left\| z_2 \right\|_2^2 =
    \left\| z_2^\text{sup} \right\|_2^2 =
    \left\| W_p z_1 \right\|_2^2 = 1,
\end{aligned}
\end{align}
where we omit stop-gradient applied to $z_2$ and $z_2^\text{sup}$ for brevity, and the equality in the second line holds due to the linearity of the L2 loss between L2-normalized features with respect to the target~\citep{lee2020mix}.
Hence, this optimization can be interpreted as mapping one view $z_1$ to an interpolated target between another view $z_2$ and the supervised target $z_2^\text{sup}$.
Intuitively, when $\alpha=1$, the model cannot determine the exact augmentation applied to $x_2$ by observing $x_1$, such that it predicts $z_2$ from $z_1$ through low-rank approximation via principal component analysis (PCA)~\cite{richemond2023edge}.
Similarly, when $\alpha=0$, the model cannot infer the exact supervised target $z_2^\text{sup}$ by observing $z_1$;
instead, it predicts $z_2^\text{sup}$ by mapping $z_1$ to the class centroid.
Here, it has been known that least squares with targets independent of each other (ignoring centering, if applied) is equivalent to linear discriminant analysis (LDA)~\cite{lee2015equivalence}, where LDA simultaneously maximizes between-class scatter and minimizes within-class scatter.
Hence, we can hypothesize that incorporating supervision into ANCL reduces intra-class variance, and the degree of reduction is controlled by $\alpha$.

To prove that incorporating supervision into ANCL reduces intra-class variance, we establish the following:
1)~the optimal predictor $W_p^*$ generates features of data with reduced intra-class variance by a factor of $\alpha$, and
2)~the optimal $W_p$ and $W$ share the same eigenspace, thereby $W$ learns to reduce intra-class variance of features.

First, we formulate the Lagrangian function of \cref{eq:loss_contrained} and take the expectation over $x_1$, $x_2$, and $x_2^\text{sup}$:
{%
\thinmuskip=.25mu 
\medmuskip=.4mu plus 2mu minus 4mu
\thickmuskip=.6mu plus 5mu 
\begin{align}
\label{eq:loss_expect}
    \mathcal{L} &= 2
    -2\alpha \cdot \tr\left( W_p^\top \mathbb{E}\left[ z_2 z_1^{\top} \right] \right)
    -2(1-\alpha) \cdot \tr\left( W_p^\top \mathbb{E}\left[ z_2^\text{sup} z_1^\top\right] \right) \nonumber \\
    &\quad + \lambda_1 \left( \tr\left( \mathbb{E}\left[ z_2 z_2^\top \right] \right) -1 \right) + \lambda_2 \left( \tr\left( \mathbb{E}\left[ z_2^\text{sup} z_2^{\text{sup} \top} \right] \right) -1 \right) \nonumber \\
    &\quad + \lambda_3 \left( \tr\left( W_p^\top W_p \mathbb{E}\left[ z_1 z_1^\top \right] \right) -1 \right), 
\end{align}
}%
where $\lambda_1$, $\lambda_2$, and $\lambda_3$ are the Lagrange multipliers.
Note that $x$ and $x^\text{sup}$ are independently sampled from the conditional distribution $P(X|Y=y)$, $x_1$ and $x_2$ are independently sampled from $P(\widetilde{X}|X=x)$, and $x_2^\text{sup}$ is sampled from $P(\widetilde{X}|X=x^\text{sup})$.

\begin{restatable}{proposition}{propositionCov}
\label{proposition:cov}
The covariance matrices of features
$\mathbb{E}\left[ z_1 z_1^\top \right]$,
$\mathbb{E}\left[ z_2 z_1^\top \right]$, and
$\mathbb{E}\left[ z_2^\text{\normalfont sup} z_1^\top \right]$
share the same eigenspace in the data space.
\end{restatable}
\begin{proof}
From \cref{assump:linear,assump:aug,assump:class,assump:cond},
\begin{align}
\begin{aligned}[b]
\label{eq:cov}
    \mathbb{E}\left[ z_1 z_1^{\top} \right] &= W \left( S_B + S_W + S_e \right)W^\top, \\
    \mathbb{E}\left[ z_2 z_1^{\top} \right] &= W \left( S_B + S_W \right)W^\top = W W^\top, \\
    \mathbb{E}\left[ z_2^\text{sup} z_1^\top \right] &= W S_B W^\top,
\end{aligned}
\end{align}
where $S_B = \frac{1}{C} \sum_y \mu_y \mu_y^{\top}$ is the inter-class covariance, $S_W = \frac{1}{C} \sum_y \Sigma_y$ is the intra-class covariance, and $S_e = \sigma^2_e I$ is the variance of the augmentation noise.
Let $S_B = V \Lambda_B V^\top$ be the eigendecomposition, where $V$ is an orthogonal matrix and $\Lambda_B$ is a diagonal matrix of the eigenvalues.
Then, $S_T = S_B + S_W$ and $S_e$ share the same eigenspace with $S_B$, as they are (scaled) identity matrices.

\begin{align}
\begin{aligned}[b]
\label{eq:cov_eig}
    \mathbb{E}\left[ z_1 z_1^{\top} \right] &= W V \left( \Lambda_B + \Lambda_W + \sigma^2_e I \right) V^\top W^\top, \\
    \mathbb{E}\left[ z_2 z_1^{\top} \right] &= W V \left( \Lambda_B + \Lambda_W \right) V^\top W^\top, \\
    \mathbb{E}\left[ z_2^\text{sup} z_1^\top \right] &= W V \Lambda_B V^\top W^\top,
\end{aligned}
\end{align}
where $\Lambda_W = I - \Lambda_B$ is the eigenvalue matrix of $S_W$.
It can be seen that the covariance matrices of features in \cref{eq:cov_eig} share the same eigenspace in the data space.
\end{proof}
\begin{figure*}[t]
\subfigure[$\alpha = 0.0$]{\includegraphics[width=0.195\textwidth, height=4cm]{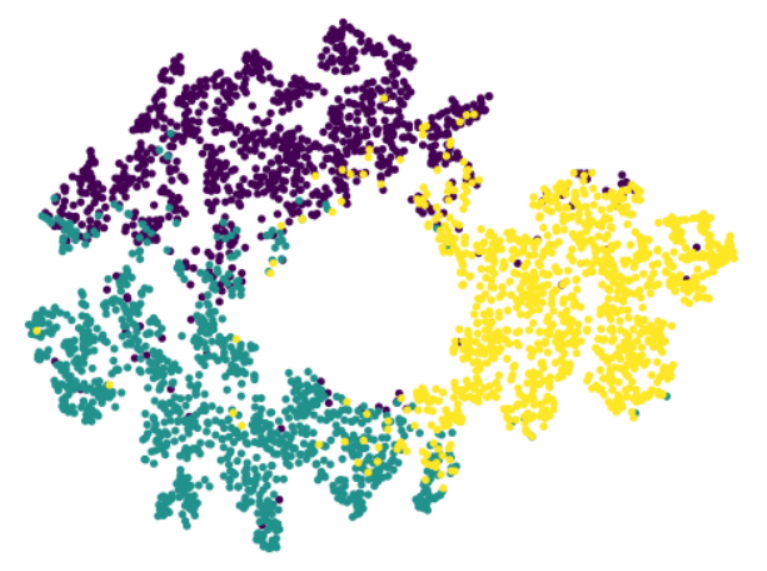}}
\subfigure[$\alpha = 0.2$]{\includegraphics[width=0.195\textwidth, height=4cm]{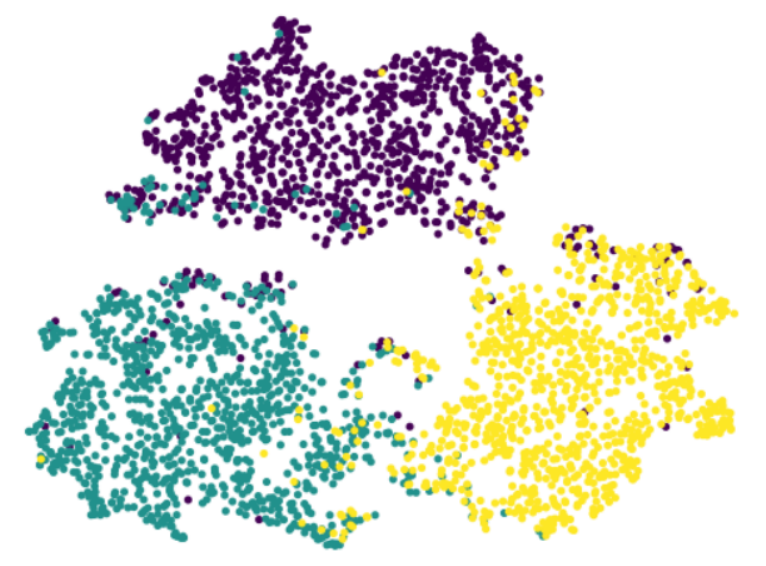}}
\subfigure[$\alpha = 0.5$]{\includegraphics[width=0.195\textwidth, height=4cm]{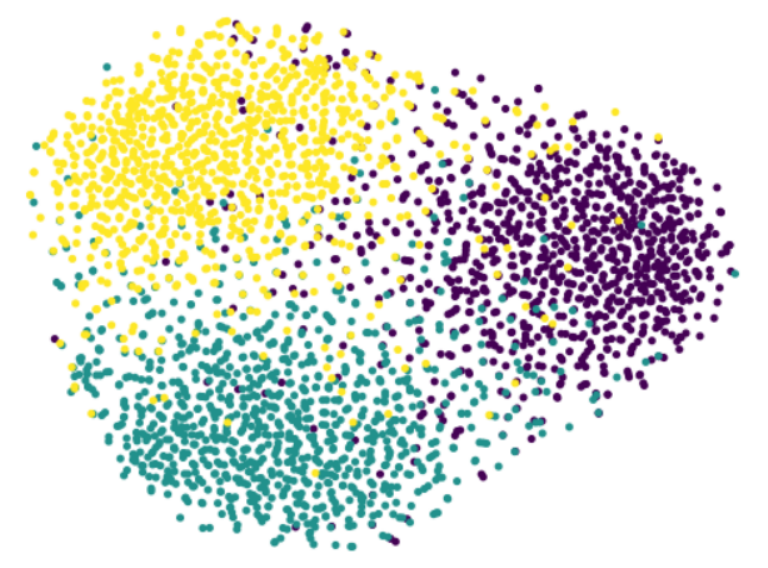}}
\subfigure[$\alpha = 0.8$]{\includegraphics[width=0.195\textwidth, height=4cm]{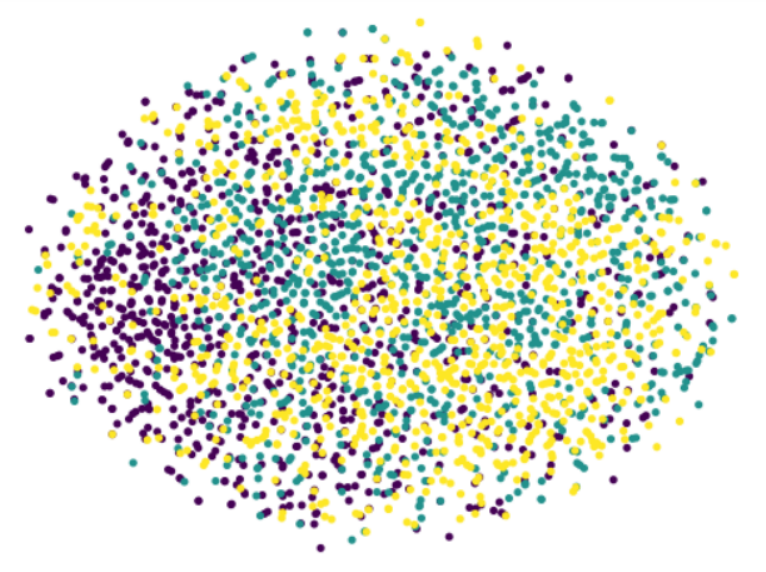}}
\subfigure[$\alpha = 1.0$]{\includegraphics[width=0.195\textwidth, height=4cm]{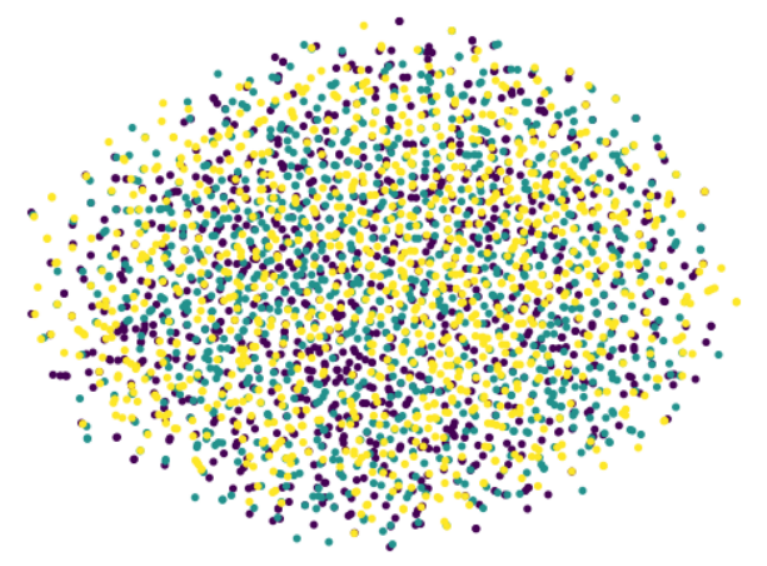}}
\vspace{-2pt}
\caption{t-SNE visualization of \textsc{SupSiam} features with different $\alpha$ on the toy dataset.}
\label{fig:subplot_tsne_toy}
\vspace{-8pt}
\end{figure*}

Then, we apply the expressions in \cref{proposition:cov} to the optimal predictor $W_p^*$ obtained from \cref{eq:loss_expect}:
\begin{restatable}{theorem}{theoremOptimalPred}
\label{thm:optimal_pred}
For an arbitrary $W$, the optimal predictor $W_p^*$ that minimizes the loss in \cref{eq:loss_expect} is given by
{%
\thinmuskip=.5mu
\medmuskip=1mu plus 2mu minus 4mu
\thickmuskip=1.3mu plus 5mu
\begin{equation*}
    W_p^*= \frac{1}{\lambda_3} WV \left( \Lambda_B + \alpha \Lambda_W \right) \left( \Lambda_B + \Lambda_W + \sigma^2_e I \right)^{-1} V^\top W^+,
\end{equation*}
}%
where $W^+$ is the Moore-Penrose inverse~\cite{penrose1955generalized}.
\end{restatable}

From \cref{thm:optimal_pred}, the optimal predictor $W_p^*$ can be interpreted through a sequence of hypothetical transformations:
1)~mapping features to the data space,
2)~eliminating the augmentation noise and reducing the intra-class variance by a factor of $\alpha$, and
3)~mapping back to the feature space.

Next, we show that $W_p^*$ and $W^*$ share the same eigenspace.
\begin{restatable}{theorem}{theoremAlign}
\label{thm:align}
The optimal predictor $W_p^*$ and the optimal model $W^*$ that minimizes the loss in \cref{eq:loss_expect} satisfy
\begin{equation*}
    W_p^{* \top} W_p^* \approx W^*W^{*\top}.
\end{equation*}
\end{restatable}
Note that \cref{thm:align} holds in self-supervised ANCL as shown in \citet{tian21, liu22gap}, and it remains valid when supervision is incorporated.

Finally, we conclude that $W^*$ learns to reduce intra-class variance, as $W_p^*$ generates features of data with reduced intra-class variance by a factor of $\alpha$ from \cref{thm:optimal_pred}, and
$W^*$ imitates this behavior according to \cref{thm:align}.

\subsection{Effect of Reducing Intra-Class Variance}
\label{sec:role}

In the proposed supervised ANCL loss in \cref{eq:loss}, the coefficient $\alpha$ adjusts the contribution of supervision:
decreasing $\alpha$ results in increasing this contribution, thereby reducing intra-class variance, as proved in \cref{sec:math}.
Ideally, when intra-class variance is too small, all data within each class converge to a single point, leading to class collapse:
data within each class become indistinguishable \citep{papyan15collapse}.
Thus, we argue that balancing the contributions of supervision and self-supervision is crucial to achieve semantically aligned yet well-distributed representations in supervised ANCL, leading to the generalization of learned representations;
intuitively, the ideal semantic latent space should retain intra-class variance to distinguish data instances.

\begin{table}[t]
\caption{\textsc{SupSiam} results with different $\alpha$ on the toy dataset and ImageNet-100 in several metrics:
the self-supervised loss in \cref{eq:ssl_loss},
the supervised loss in \cref{eq:sup_loss},
the intra-class variance,
the relative intra-class variance (\%), and
the accuracy of $k$-NN and linear probing (\%).
For the accuracies, the \textbf{best results} are highlighted in bold and the \underline{second-best results} are underlined.}
\label{table:metric}
\begin{center}
\resizebox{\columnwidth}{!}{%
\begin{tabular}{c|ccccc}
\toprule
$\alpha$ & ($\ell_{\text{ssl}}$, $\ell_{\text{sup}}$) & $\widetilde{S}_W$ & $\widetilde{S}_W / \widetilde{S}_T$ & $k$-NN & Linear \\
\midrule
Toy dataset \\
\cmidrule{1-1}
0.0 & (-0.7115, -0.4212) & 0.338 & 33.94 & 60.51 & 60.78 \\
0.2 & (-0.7423, -0.4140) & 0.363 & 36.36 & \textbf{61.35} & \underline{61.58} \\ 
0.5 & (-0.8116, -0.1654) & 0.799 & 79.97 & \underline{61.18} & \textbf{61.89} \\
0.8 & (-0.8473, -0.0231) & 0.971 & 97.14 & 55.93 & 61.18 \\ 
1.0 & (-0.8519, -0.0024) & 0.997 & 99.70 & 38.02 & 45.91 \\
\cmidrule{1-6}
ImageNet-100 \\
\cmidrule{1-1}
0.0 & (-0.9048, -0.8932) & 0.070 & 7.01 & 80.79 & 85.92  \\
0.2 & (-0.9231, -0.9096) & 0.057 & 5.72 & \underline{82.72} & \underline{86.85}  \\
0.5 & (-0.9321, -0.8823) & 0.108 & 10.79 & \textbf{82.89} & \textbf{87.31}  \\
0.8 & (-0.9349, -0.5118) & 0.515 & 51.58 & 80.19 & 86.65  \\
1.0 & (-0.9290, -0.2341) & 0.743 & 74.53 & 75.23 & 82.15  \\
\bottomrule
\end{tabular}
}
\end{center}
\vspace{-10pt}
\end{table}

To verify our claim, we conduct an experiment on a synthetic toy dataset with three classes, each following a Gaussian distribution by training \textsc{SupSiam} models with varying $\alpha$.
The details of the toy dataset and \textsc{SupSiam} models are described in \cref{sec:toy_setting}.
After training, we compare the self-supervised and supervised training losses ($\ell_{\text{ssl}}$ and $\ell_{\text{sup}}$), the absolute and relative intra-class variance of latent features ($\widetilde{S}_W$ and $\widetilde{S}_W / \widetilde{S}_T$), and the accuracy of $k$-nearest neighbors ($k$-NN) and linear probing (Linear) in \cref{table:metric}.
Here, $\widetilde{S}_W$ and $\widetilde{S}_T$ represent the empirical intra-class variance and the total variance, respectively:
\begin{equation}
    \widetilde{S}_W = \mathbb{E}_{y,z} \left[ \| z - \tilde{\mu}_{y} \|_2^2 \right], \;
    \widetilde{S}_T = \mathbb{E}_z \left[ \| z - \tilde{\mu} \|_2^2 \right],
\end{equation}
where $\tilde{\mu}_y$ is the $y$-th class mean and $\tilde{\mu}$ is the total mean of features, and the expectation is taken over training dataset.

In \cref{table:metric}, $\ell_{\text{ssl}}$ decreases while $\ell_{\text{sup}}$ increases as $\alpha$ increases, which confirms that the contribution of each loss is adjusted as expected.
Additionally, the intra-class variance is proportional to $\alpha$, as proved in \cref{sec:math}.
However, the accuracy of $k$-NN and linear probing exhibits different trends, with the best accuracy achieved when $\alpha$ is between 0 and 1.
This supports our claim that while incorporating supervision into ANCL aids in learning semantically aligned representations, excessively reducing intra-class variance may hinder the generalization of learned representations, resulting in diminishing performance on unseen test data.

\Cref{fig:subplot_tsne_toy} visualizes the feature space via t-SNE \citep{maaten2008visualizing}.
When $\alpha=0.5$, the class distributions are well-separated while retaining intra-class variance.
Decreasing $\alpha$ results in more densely clustered results by skewing the feature space, which might be detrimental to generalization;
\eg, the model is overconfident in its predictions for downstream classification tasks.
Conversely, increasing $\alpha$ leads to mixed class distributions,
impairing classification.

\begin{table}[t]
\caption{Transfer learning results
on toy downstream datasets with different means and varying scale of covariance $\sigma$, with \textsc{SupSiam}-pretraining on the toy dataset.
For each scenario, the \textbf{best results} are in bold and the \underline{second-best results} are underlined.}
\label{table:toy_transfer}
\begin{center}
\resizebox{0.95\columnwidth}{!}{\begin{tabular}{c|ccc|ccc}
\toprule
\multirow{2}{*}{$\alpha$} & \multicolumn{3}{c}{Interpolation} & \multicolumn{3}{|c}{Extrapolation} \\
\cmidrule(lr){2-4} \cmidrule(lr){5-7}
& $\sigma = 0.2$ & $\sigma = 0.5$ & $\sigma = 0.8$ &$\sigma = 0.2$ & $\sigma = 0.5$ & $\sigma = 0.8$ \\ 
\midrule
0.0 & 43.60 & 37.44 & 35.40 & 96.67 & 83.76 & 74.42 \\
0.2 & 44.02 & 37.24 & 35.31 & \underline{97.13} & \underline{84.71} & \underline{75.09} \\
0.5 & \textbf{44.25} & \textbf{37.96} & \underline{35.69} & \textbf{97.60} & \textbf{85.87} & \textbf{76.00} \\
0.8 & \underline{44.24} & \underline{37.65} & \textbf{35.98} & 97.07 & 83.69 & 73.73 \\
1.0 & 40.40 & 36.60 & 35.67 & 75.84 & 59.67 & 53.00 \\
\bottomrule
\end{tabular}}
\end{center}
\vspace{-4pt}
\end{table}

To assess the transferability of learned representations, we conduct transfer learning scenarios in \cref{table:toy_transfer}.
Specifically, we consider three downstream classes, where their means are either interpolated or extrapolated from the pretraining classes, and the scale of the covariance matrix of downstream classes is adjusted by $\sigma$ to control the difficulty of downstream tasks.
As shown in \cref{table:toy_transfer}, supervised ANCL consistently outperforms self-supervised ANCL ($\alpha=1$) across all scenarios, highlighting the effectiveness of incorporating supervision into ANCL.
Moreover, the best performance is achieved when $0 < \alpha < 1$, suggesting that balancing the contributions of supervision and self-supervision is crucial, \ie, excessively reducing intra-class variance is detrimental to representation learning.

Next, to confirm the scalability of our observations to real-world scenarios, we conduct a similar experiment on ImageNet-100 \citep{deng2009imagenet, tian19cmc} by replacing the encoder with ResNet-50 \citep{he2016deep} and the projector and predictor with MLPs, respectively.
As shown in the bottom of \cref{table:metric}, the observations remain mostly consistent;
although both supervised loss and intra-class variance slightly decrease when $\alpha$ increases from 0.0 to 0.2, we conjecture that this is due to the non-linearity of the optimization.
These results further support our claim that balancing the contributions of supervision and self-supervision is crucial for the generalization of representations learned via supervised ANCL.

\begin{table}[t]
\caption{Transfer learning results on fine-grained classification datasets, where the model is \textsc{SupSiam}-pretrained with different $\alpha$ on ImageNet-100.
For each dataset, the \textbf{best results} are in bold and the \underline{second-best results} are underlined.}
\label{table:real_fine_grained}
\begin{center}
\begin{tabular}{l|ccc}
\toprule
$\alpha$ & CUB200 & Dogs & Pets \\
\midrule
0.0 & 41.46 & 61.51 & 80.09 \\
0.2 & 42.07 & \underline{64.28} & \underline{82.27} \\
0.5 & \textbf{43.48} & \textbf{64.65} & \textbf{82.38} \\
0.8 & \underline{42.16} & 62.94 & 81.76 \\
1.0 & 36.10 & 54.57 & 75.13 \\
\bottomrule
\end{tabular}
\end{center}
\vspace{-4pt}
\end{table}

\begin{figure*}[t]
\subfigure[$\alpha = 0.0$]{\includegraphics[width=0.195\textwidth, height=4cm]{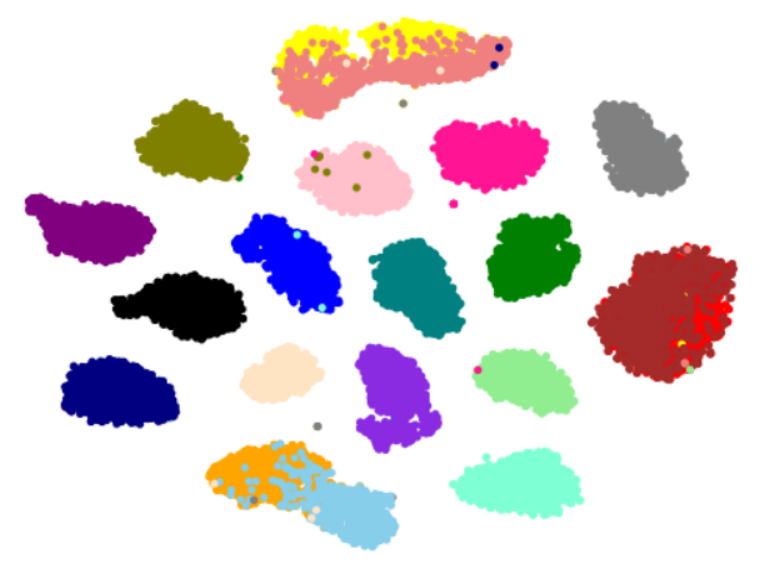}}
\subfigure[$\alpha = 0.2$]{\includegraphics[width=0.195\textwidth, height=4cm]{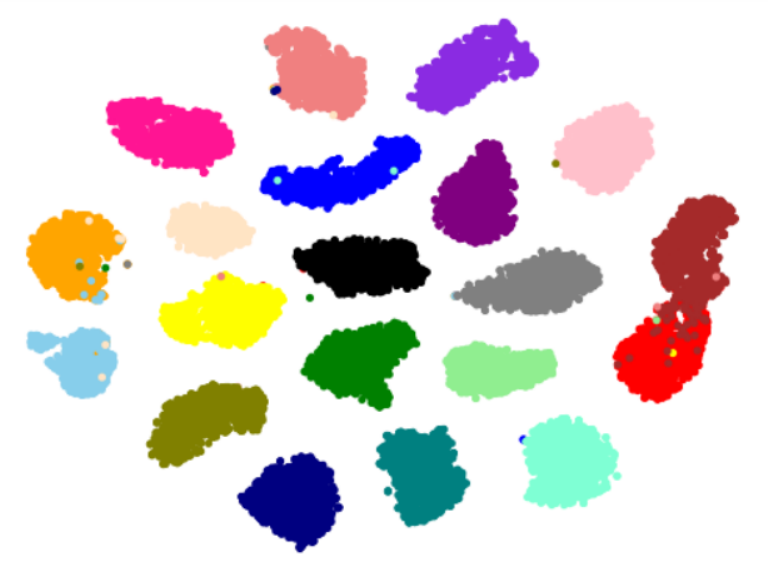}}
\subfigure[$\alpha = 0.5$]{\includegraphics[width=0.195\textwidth, height=4cm]{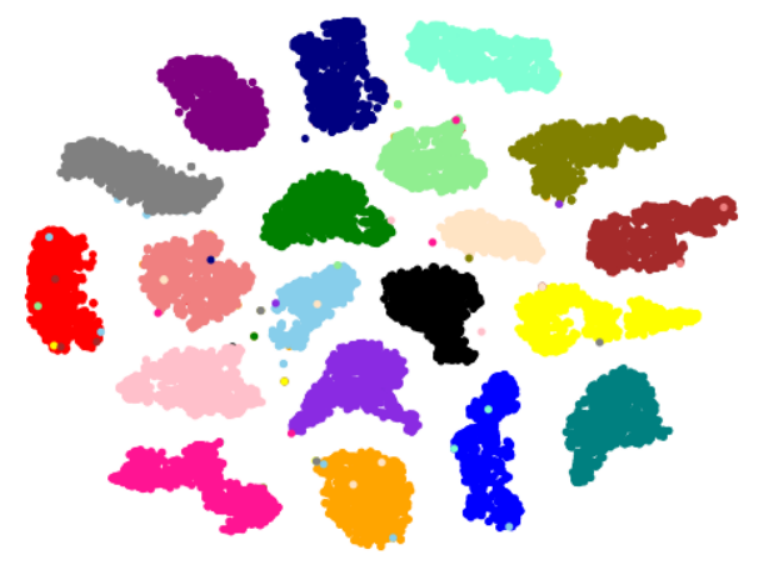}}
\subfigure[$\alpha = 0.8$]{\includegraphics[width=0.195\textwidth, height=4cm]{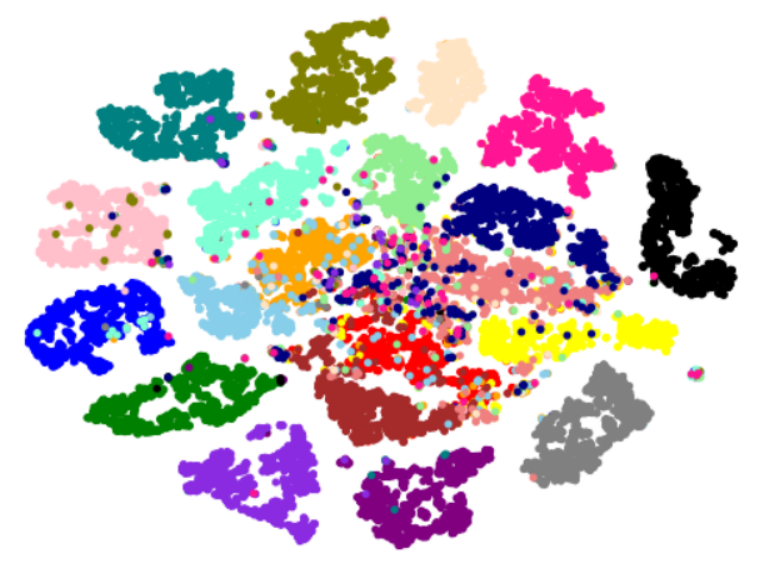}}
\subfigure[$\alpha = 1.0$]{\includegraphics[width=0.195\textwidth, height=4cm]{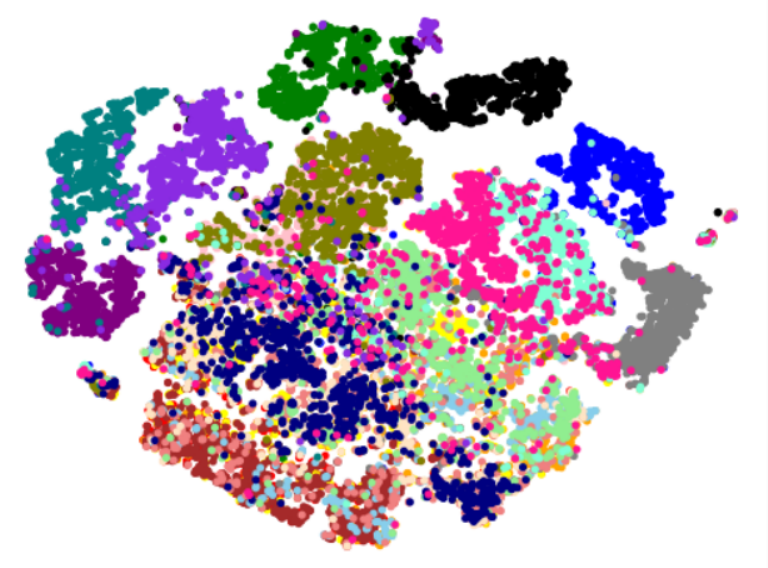}}
\vspace{-2pt}
\caption{
t-SNE visualization of \textsc{SupSiam} features with different $\alpha$ on 15 dog and 5 bird classes from ImageNet-100.
}
\label{fig:subplot_tsne}
\vspace{-8pt}
\end{figure*}

Similar to the toy experiments, we evaluate the transferability of learned representations in real-world scenarios by conducting transfer learning experiments.
Specifically, we apply linear probing to the \textsc{SupSiam}-pretrained models on downstream datasets for fine-grained classification tasks, including CUB-200-2011 \citep{welinder2010cub}, Stanford Dogs \citep{khosla11dog}, and Oxford-IIIT Pets \citep{parkhi12pet}.
As shown in \cref{table:real_fine_grained}, the transfer learning performance exhibits trends similar to those in \cref{table:metric}:
incorporating supervision into ANCL is beneficial, and balancing the contributions of supervision and self-supervision improves the generalization of representations.

To further elucidate the effect of $\alpha$ in real-world scenarios, we present t-SNE visualizations of latent features from 20 classes, consisting of 15 dogs and 5 birds, subsampled from ImageNet-100.
As shown in \cref{fig:subplot_tsne}, classes overlap when no supervision is provided, \ie, when $\alpha=1$, and the latent features form more compact clusters as $\alpha$ decreases.
Notably, some dog classes (\eg, ``Doberman'' and ``Rottweiler'') overlap when $\alpha$ is small, around 0.0 and 0.2, while they are well-separated when $\alpha=0.5$.
This implies that excessively reducing intra-class variance with small $\alpha$ might result in collapsing fine-grained classes, which could be detrimental to downstream tasks.

\section{Experiment}

In this section, we provide experimental results across various datasets and tasks to demonstrate the effectiveness of supervision in ANCL.
We also compare CL methods to confirm that ANCL performance is competitive to CL.
Detailed experimental settings are provided in \cref{sec:pretrain_app}.

\subsection{Pretraining}

We consider two ANCL methods, \textsc{SimSiam} \citep{chen21siamese} and \textsc{BYOL} \citep{grill20bootstrap}, as our baselines, along with their supervised variations, \textsc{SupSiam} and \textsc{SupBYOL}, as our proposed methods.
Additionally, we compare two CL methods, \textsc{SimCLR} \citep{chen20simple} and \textsc{MoCo-v2} \citep{chen20improved}, and their supervised variations, \textsc{SupCon} \citep{khosla20supcon} and \textsc{SupMoCo} \citep{majumder21supmoco}.
Each model consists of a ResNet-50 encoder \citep{he2016deep} followed by a 2-layer MLP projector and predictor, except for \textsc{SimSiam} and \textsc{SupSiam}, which utilize a 3-layer MLP projector following the original configuration by \citet{chen21siamese}.
We pretrain models on ImageNet-100 \citep{deng2009imagenet, tian19cmc} for 200 epochs with a batch size of 128.
For data augmentation, we apply random crop, random horizontal flip, color jitter, random grayscale, and Gaussian blur, following \citet{chen20simple}.
For methods utilizing the target pool, we set the size of the target pool $|Q|$ to 8192 and obtain the supervised target $z_2^\text{sup}$ by sampling and averaging all positives in the target pool%
\footnote{%
Our proposed methods are robust to the number of positives from the target pool, as shown in \cref{table:numpos}.
}
unless otherwise stated.
The coefficient $\alpha$ adjusting the contribution of the self-supervised and supervised loss is 0.5, unless otherwise stated.
We repeat all experiments with three pretrained models with different random seeds and report the average performance.

\begin{table}[t]
\caption{Top-1 linear probing accuracy on ImageNet-100 and transfer learning performance on VOC object detection.
The \textbf{best results} are in bold and the \underline{second-best results} are underlined.
Our proposed methods are marked with $\dagger$.}
\label{table:linear-eval&object_detection}
\begin{center}
\resizebox{0.95\columnwidth}{!}{
\begin{tabular}{l|cccc}
\toprule
Dataset & \small{ImageNet-100}  & \multicolumn{2}{c}{VOC} \\
\cmidrule(rl){1-1}\cmidrule(rl){2-2}\cmidrule(rl){3-4}
Method & Top-1  & AP & AP50 \\
\midrule
\textsc{SimCLR}              & 77.35 & 52.06 {\scriptsize $\pm$ 0.31}  & 78.70 {\scriptsize $\pm$ 0.16} \\
\textsc{SupCon}       & \underline{87.40} & 52.53 {\scriptsize $\pm$ 0.47} & 79.44 {\scriptsize $\pm$ 0.21} \\
\cmidrule(rl){1-4}
\textsc{MoCo-v2}          & 78.37 & 52.68 {\scriptsize $\pm$ 0.04}  & 79.08 {\scriptsize $\pm$ 0.24} \\
\textsc{SupMoCo}                & 86.33 & 52.67 {\scriptsize $\pm$ 0.04} & 79.52 {\scriptsize $\pm$ 0.15} \\
\cmidrule(rl){1-4}
\textsc{SimSiam}          & 82.15 & 53.56 {\scriptsize $\pm$ 0.10} & 79.82 {\scriptsize $\pm$ 0.10} \\
\textsc{SupSiam}$^\dagger$               & 87.31 & \textbf{53.89 {\scriptsize $\pm$ 0.26} } & \textbf{80.28 {\scriptsize $\pm$ 0.06}} \\
\cmidrule(rl){1-4}
\textsc{BYOL}                & 84.93 & 53.54 {\scriptsize $\pm$ 0.04} & 79.57 {\scriptsize $\pm$ 0.01} \\
\textsc{SupBYOL}$^\dagger$           & \textbf{87.43} & \underline{53.69 {\scriptsize $\pm$ 0.24}} & \underline{80.26 {\scriptsize $\pm$ 0.17}} \\
\bottomrule
\end{tabular}
}
\end{center}
\vspace{-20pt}
\end{table}

\begin{table*}[t]
\caption{Transfer learning via linear evaluation results on various downstream datasets, where models are pretrained on ImageNet-100.
\textit{CL}, \textit{Sup.}, \textit{EMA} stand for the cases when
negative samples are considered, labels are used for pretraining, and the momentum network is adopted, respectively.
\textit{Avg. Rank} represents the average performance ranking across all datasets.
For each dataset, the \textbf{best results} are in bold and the \underline{second-best results} are underlined.
Our proposed methods are marked with $\dagger$.}
\label{table:transfer}
\begin{center}
\resizebox{\textwidth}{!}{%
\begin{tabular}{l|ccc|cccccccccccc}
\toprule
Method & \textit{CL} & \textit{Sup} & \textit{EMA} & \textit{Avg.Rank} & CIFAR10 & CIFAR100 & DTD & Food & MIT67 & SUN397 & Caltech & CUB200 & Dogs & Flowers & Pets \\
\midrule
\textsc{SimCLR} &\cmark & & & 7.00 & 84.69 & 62.86 & 64.18 & 60.91 & 61.81 & 47.10 & 77.89 & 28.76 & 44.33 & 84.30 & 65.10 \\
\textsc{SupCon} & \cmark & \cmark & & 4.73 & 88.82 & 68.89 & 65.18 & 59.34 & 63.76 & 50.09 & 87.30 & 35.84 & 61.68 & 89.05 & 80.12 \\
\cmidrule{1-16}
\textsc{MoCo-v2} & \cmark & &\cmark   & 7.82 & 83.43 & 61.54 & 61.81 & 57.36 & 59.55 & 45.07 & 77.26 & 27.79 & 46.67 & 82.35 & 68.52 \\
\textsc{SupMoCo} & \cmark & \cmark & \cmark & 4.09 & 89.05 & 69.29 & 65.44 & 59.04 & 63.46 & 50.05 & 87.54 & 37.75 & 62.80 & 89.69 & 80.81 \\
\cmidrule{1-16}
\textsc{SimSiam} & & & & 4.91 &  87.28 & 66.41 & 66.06 & 63.44 & 64.68 & 50.69 & 85.00 & 36.10 & 54.57 & 88.38 & 75.13 \\
\textsc{SupSiam}$^\dagger$ & & \cmark & & \underline{2.27} & \underline{89.95} & \underline{70.88} & 66.51 & 61.46 & 64.45 & \underline{51.50} & \textbf{88.86} & \textbf{43.48} & \underline{64.65} & \underline{90.27} & \underline{82.38} \\
\cmidrule{1-16}
\textsc{BYOL}& & & \cmark & 3.82 & 88.26 & 68.08 & \textbf{67.52} & \textbf{64.63} & \underline{65.70} & 51.21 & 85.85 & 37.10 & 57.80 & 88.14 & 78.78 \\
\textsc{SupBYOL}$^\dagger$ & & \cmark & \cmark & \textbf{1.36} & \textbf{90.85} & \textbf{72.04} & \underline{67.38} & \underline{64.58} & \textbf{66.64} & \textbf{52.95} & \underline{88.79} & \underline{43.24} & \textbf{65.02} & \textbf{91.09} & \textbf{82.68} \\

\bottomrule
\end{tabular}%
}
\end{center}
\vspace{-11pt}
\end{table*}

\subsection{Linear Evaluation}

We evaluate the quality of representations on the pretrained distribution through a comparison of linear probing performance on ImageNet-100.
Specifically, we take the pretrained and frozen backbone, and train a linear classifier on top of it, following the common protocol in prior works \citep{chen20simple, chen20improved, grill20bootstrap, chen21siamese}.

As shown in \cref{table:linear-eval&object_detection}, incorporating supervision into ANCL enhances linear probing performance on the pretraining dataset.
This suggests that representations learned with supervision more effectively encode the semantic information of the pretrained data distribution.

\subsection{Object Detection}

To assess the generalizability beyond classification tasks, we evaluate pretraining methods on an object detection task. 
Following \citet{he20contrast}, we initialize Faster R-CNN \citep{ren15rcnn} with each pretrained model and fine-tune it on the VOC07+12 training dataset \citep{everingham10voc}.
We measure performance using the COCO evaluation metrics \citep{lin14coco} on the VOC07 test dataset.

As shown on the right side of \cref{table:linear-eval&object_detection}, incorporating supervision into ANCL improves object detection performance, resulting in the best overall performance.
In contrast, the performance gain from supervision in CL is marginal or often detrimental, which aligns with the findings from prior works \citep{khosla20supcon}.
This suggests that supervised ANCL yields more generalizable representations, with the potential to achieve superior performance across various downstream tasks.

\begin{table*}[t]
\caption{Few-shot classification accuracy averaged over 2000 episodes on various datasets, where models are pretrained on ImageNet-100.
\textit{CL}, \textit{Sup}, \textit{EMA} stand for the cases when
negative samples are considered, labels are used for pretraining, and the momentum network is adopted, respectively.
\textit{Avg.Rank} represents the average performance ranking across all datasets.
For each dataset, the \textbf{best results} are in bold and the \underline{second-best results} are underlined.
Our proposed methods are marked with $\dagger$.} 
\label{table:linear-fewshot}
\begin{center}
\resizebox{\textwidth}{!}{%
\begin{tabular}{l|ccc|ccccccccc}
\toprule
Method & \textit{CL} & \textit{Sup} & \textit{EMA} & \textit{Avg.Rank} & Aircraft & CUB200 & FC100 & Flowers & Fungi & Omniglot & DTD & Traffic Signs \\
\midrule
5-way 1-shot \\ 
\cmidrule{1-13}
\textsc{SimCLR} &\cmark & &  & 7.25 & 29.22 {\scriptsize$\pm$ 0.34} & 40.61 {\scriptsize$\pm$ 0.43} & 35.53 {\scriptsize$\pm$ 0.37} & 68.26 {\scriptsize$\pm$ 0.50} & 42.44 {\scriptsize$\pm$ 0.44} & 70.46 {\scriptsize$\pm$ 0.54} & 55.43 {\scriptsize$\pm$ 0.45} & 48.33 {\scriptsize$\pm$ 0.43} \\
\textsc{SupCon} & \cmark & \cmark &  & 3.63 & 31.44 {\scriptsize$\pm$ 0.35} & 48.75 {\scriptsize$\pm$ 0.49} & \underline{45.32 {\scriptsize$\pm$ 0.41}} & 77.99 {\scriptsize$\pm$ 0.44} & 47.42 {\scriptsize$\pm$ 0.45} & 80.66 {\scriptsize$\pm$ 0.45} & 57.57 {\scriptsize$\pm$ 0.47} & 68.66 {\scriptsize$\pm$ 0.47} \\ 
\cmidrule{1-13}
\textsc{MoCo-v2} & \cmark & &\cmark  & 7.38 & 25.54 {\scriptsize$\pm$ 0.28} & 41.24 {\scriptsize$\pm$ 0.46} & 36.73 {\scriptsize$\pm$ 0.36} & 66.48 {\scriptsize$\pm$ 0.50} & 41.84 {\scriptsize $\pm$ 0.44} & 71.12 {\scriptsize $\pm$ 0.51} & 54.75 {\scriptsize $\pm$ 0.46} & 51.05 {\scriptsize $\pm$ 0.43} \\
\textsc{SupMoCo} & \cmark & \cmark & \cmark & 3.25 & 31.12 {\scriptsize $\pm$ 0.35} & 49.04 {\scriptsize $\pm$ 0.49} & 44.13 {\scriptsize $\pm$ 0.41} & \underline{78.90 {\scriptsize $\pm$ 0.43}} & 47.12 {\scriptsize $\pm$ 0.45} & \underline{83.43 {\scriptsize $\pm$ 0.42}} & 56.62 {\scriptsize $\pm$ 0.46} & \textbf{71.17 {\scriptsize $\pm$ 0.47}} \\
\cmidrule{1-13}
\textsc{SimSiam} & & & & 5.00 & 30.67 {\scriptsize $\pm$ 0.35} & 45.06 {\scriptsize $\pm$ 0.47} & 41.51 {\scriptsize $\pm$ 0.40} & 75.68 {\scriptsize $\pm$ 0.47} & 45.22 {\scriptsize $\pm$ 0.46} & 74.64 {\scriptsize $\pm$ 0.50} & 58.28 {\scriptsize $\pm$ 0.47} & 60.03 {\scriptsize $\pm$ 0.45} \\
\textsc{SupSiam}$^\dagger$ & &\cmark &  & \textbf{1.88} & \textbf{33.12 {\scriptsize $\pm$ 0.37}} & \textbf{49.58 {\scriptsize $\pm$ 0.49}} & \textbf{45.56 {\scriptsize $\pm$ 0.41}} & 78.12 {\scriptsize $\pm$ 0.44} & \underline{47.74{\scriptsize $\pm$ 0.46}} & \textbf{84.02 {\scriptsize $\pm$ 0.41}} & 58.06 {\scriptsize $\pm$ 0.48} & \underline{71.00 {\scriptsize $\pm$ 0.48}} \\
\cmidrule{1-13}
\textsc{BYOL} & & &\cmark  & 5.63 & 26.38 {\scriptsize $\pm$ 0.30} & 46.45 {\scriptsize $\pm$ 0.49} & 40.92 {\scriptsize $\pm$ 0.40} & 74.27 {\scriptsize $\pm$ 0.47} & 45.96 {\scriptsize $\pm$ 0.46} & 68.13 {\scriptsize $\pm$ 0.52} & \underline{59.75 {\scriptsize $\pm$ 0.48}} & 57.44 {\scriptsize $\pm$ 0.46} \\
\textsc{SupBYOL}$^\dagger$ & & \cmark & \cmark  & \underline{2.00} & \underline{32.66 {\scriptsize $\pm$ 0.37}} & \underline{49.26 {\scriptsize $\pm$ 0.48}} & 45.28 {\scriptsize $\pm$ 0.41} & \textbf{78.94 {\scriptsize $\pm$ 0.43}} & \textbf{47.81 {\scriptsize $\pm$ 0.46}} & 82.62 {\scriptsize $\pm$ 0.44} & \textbf{59.98 {\scriptsize $\pm$ 0.48}} & 70.34 {\scriptsize $\pm$ 0.48} \\
\midrule
5-way 5-shot \\
\cmidrule{1-13}
\textsc{SimCLR} &\cmark & & & 7.13 & 39.21 {\scriptsize $\pm$ 0.44} & 54.33 {\scriptsize $\pm$ 0.45} & 50.96 {\scriptsize $\pm$ 0.37} & 86.98 {\scriptsize $\pm$ 0.30} & 59.40 {\scriptsize $\pm$ 0.47} & 86.72 {\scriptsize $\pm$ 0.35} & 73.95 {\scriptsize $\pm$ 0.36} & 69.27 {\scriptsize $\pm$ 0.40} \\
\textsc{SupCon} & \cmark & \cmark &  & 3.38 & 44.63 {\scriptsize $\pm$ 0.44} & 64.99 {\scriptsize $\pm$ 0.46} & 64.04 {\scriptsize $\pm$ 0.39} & \underline{92.80 {\scriptsize $\pm$ 0.23}} & \textbf{66.75 {\scriptsize $\pm$ 0.47}} & 93.36 {\scriptsize $\pm$ 0.25} & 75.60 {\scriptsize $\pm$ 0.36} & 85.93 {\scriptsize $\pm$ 0.36} \\
\cmidrule{1-13}
\textsc{MoCo-v2} & \cmark & &\cmark   & 7.50 & 32.84 {\scriptsize $\pm$ 0.35} & 53.42 {\scriptsize $\pm$ 0.47} & 52.70 {\scriptsize $\pm$ 0.36} & 84.72 {\scriptsize $\pm$ 0.32} & 57.54 {\scriptsize $\pm$ 0.48} & 87.74 {\scriptsize $\pm$ 0.34} & 72.66 {\scriptsize $\pm$ 0.37} & 71.93 {\scriptsize $\pm$ 0.39} \\
\textsc{SupMoCo} & \cmark & \cmark & \cmark & \underline{2.63} & 44.43 {\scriptsize $\pm$ 0.44} & 65.63 {\scriptsize $\pm$ 0.46} & 64.30 {\scriptsize $\pm$ 0.39} & \textbf{93.35 {\scriptsize $\pm$ 0.21}} & \underline{66.64 {\scriptsize $\pm$ 0.47}} & \textbf{94.77 {\scriptsize $\pm$ 0.22}} & 74.73 {\scriptsize $\pm$ 0.36} & \textbf{87.64 {\scriptsize $\pm$ 0.34}} \\
\cmidrule{1-13}
\textsc{SimSiam} & & & & 5.25 & 40.34 {\scriptsize $\pm$ 0.44} & 60.66 {\scriptsize $\pm$ 0.48} & 58.68 {\scriptsize $\pm$ 0.38} & 91.04 {\scriptsize $\pm$ 0.26} & 62.19 {\scriptsize $\pm$ 0.49} & 88.92 {\scriptsize $\pm$ 0.32} & 76.22 {\scriptsize $\pm$ 0.36} & 79.50 {\scriptsize $\pm$ 0.39} \\
\textsc{SupSiam}$^\dagger$ & &\cmark & &  \textbf{2.38} & \textbf{45.98 {\scriptsize $\pm$ 0.47}} & \underline{66.70 {\scriptsize $\pm$ 0.45}} & \underline{64.54 {\scriptsize $\pm$ 0.39}} & 92.42 {\scriptsize $\pm$ 0.23} & 66.61 {\scriptsize $\pm$ 0.48} & \underline{94.38 {\scriptsize $\pm$ 0.23}} & 76.43 {\scriptsize $\pm$ 0.36} & \underline{86.88 {\scriptsize $\pm$ 0.36}} \\
\cmidrule{1-13}
\textsc{BYOL} & & &\cmark  & 5.38 & 35.30 {\scriptsize $\pm$ 0.40} & 60.96 {\scriptsize $\pm$ 0.49} & 59.33 {\scriptsize $\pm$ 0.38} & 90.38 {\scriptsize $\pm$ 0.26} & 63.12 {\scriptsize $\pm$ 0.49} & 85.68 {\scriptsize $\pm$ 0.35} & \textbf{77.60 {\scriptsize $\pm$ 0.36}} & 77.07 {\scriptsize $\pm$ 0.40} \\
\textsc{SupBYOL}$^\dagger$ & & \cmark & \cmark  & \textbf{2.38} & \underline{45.81 {\scriptsize $\pm$ 0.48}} & \textbf{66.72 {\scriptsize $\pm$ 0.46}} & \textbf{65.72 {\scriptsize $\pm$ 0.38}} & 92.78 {\scriptsize $\pm$ 0.22} & 66.47 {\scriptsize $\pm$ 0.48} & 94.06 {\scriptsize $\pm$ 0.24} & \underline{77.57 {\scriptsize $\pm$ 0.36}} & 86.21 {\scriptsize $\pm$ 0.37} \\
\bottomrule
\end{tabular}
}
\end{center}
\vspace{-11pt}
\end{table*}

\subsection{Transfer Learning via Linear Evaluation}

For transfer learning, we evaluate the top-1 accuracy across 11 downstream datasets: CIFAR10/CIFAR100 \citep{krizhevsky2009learning}, DTD \citep{cimpoi14dtd}, Food \citep{bossard14food}, MIT67 \citep{quattoni09mit}, SUN397 \citep{xiao10sun}, Caltech101 \citep{fei04caltech}, CUB200 \citep{welinder2010cub}, Dogs \citep{khosla11dog, deng2009imagenet}, Flowers \citep{nilsback08flower}, and Pets \citep{parkhi12pet}, where detailed information is described in \cref{sec:dataset}.
For evaluation, we follow the linear probing protocol for transfer learning in prior works \citep{kornblith19do, lee21augself}.

As shown in \cref{table:transfer}, incorporating supervision improves performance across all pretraining methods.
Among them, supervised ANCL methods achieve the best performance: \textsc{SupBYOL} and \textsc{SupSiam} outperform others on 9 out of 11 datasets, demonstrating the superiority of supervised ANCL.
Between supervised ANCL methods, \textsc{SupBYOL} exhibits better performance than \textsc{SupSiam} in terms of the average rank, which might be due to the effect of momentum network.
Notably, while the performance gain from incorporating supervision into ANCL is relatively small compared to CL because the self-supervised versions of ANCL already exhibit strong performance, we observe a significant improvement on fine-grained datasets, such as CUB200, Dogs, and Pets.
This suggests that learning semantically aligned representations while retaining intra-class variance in ANCL is crucial for recognizing fine-grained information.

\subsection{Few-Shot Classification}

To assess the generalizability of learned representations under limited conditions, we conduct transfer learning experiments on few-shot classification tasks following the linear probing protocol for few-shot learning in \citet{lee21augself}.
We evaluate the accuracy of 5-way 1-shot and 5-way 5-shot scenarios over 2000 episodes across 8 downstream datasets:
Aircraft \citep{maji13aircraft}, CUB200 \citep{welinder2010cub}, FC100 \citep{oreshkin18fc}, Flowers \citep{nilsback08flower}, Fungi \citep{schroeder18fungi}, Omniglot \citep{lake15omni}, DTD \citep{cimpoi14dtd}, and Traffic Signs \citep{houben13traffic}.
\Cref{table:linear-fewshot} shows a similar trend to other experiments that incorporating supervision improves both CL and ANCL, while supervised ANCL achieves the best performance in most cases.

\begin{table*}[t]
\caption{Transfer learning via linear evaluation results on various downstream datasets, where the model is \textsc{SupSiam}-pretrained with different target pool design on ImageNet-100.
\textit{Avg} represents the average performance across each dataset.
For each dataset, the \textbf{best results} are in bold and the \underline{second-best results} are underlined.
}
\label{table:pool-design}
\begin{center}
\resizebox{\textwidth}{!}{%
\begin{tabular}{c|c|cccccccccccc}
\toprule
Target Pool & Size & \textit{Avg} & CIFAR10 & CIFAR100 & DTD & Food & MIT67 & SUN397 & Caltech & CUB200 & Dogs & Flowers & Pets \\
\midrule
\text{Class-agnostic} & 8192 & \underline{70.40} & \underline{89.95} & \underline{70.88} & 66.51 & 61.46 & 64.45 & \underline{51.50} & 88.86 & \textbf{43.48} & 64.65 & \textbf{90.27} & 82.38 \\
\text{Class-wise} & 80 $\times$ 100 & 70.21 & 89.78 & 70.58 & 66.49 & \textbf{61.54} & 64.85 & 51.22 & 88.64 & 42.68 & 65.03 & \underline{89.86} & 81.61 \\
\text{Class-wise} & 20 $\times$ 100 & \textbf{70.44} & \textbf{90.02} & \textbf{71.07} & \underline{66.92} & \underline{61.49} & \underline{65.11} & 51.15 & 88.67 & 43.19 & \textbf{65.16} & 89.27 & \underline{82.39} \\
\text{Class-wise} & 5 $\times$ 100 & 70.27 & 89.67 & \underline{70.88} & 66.17 & 61.32 & 64.30 & 51.49 & \underline{88.96} & 42.80 & 64.82 & \underline{89.86} & \textbf{82.75} \\
\text{Class-wise} & 1 $\times$ 100 & 70.23 & 89.70 & 70.73 & 66.06 & 61.45 & 64.82 & 51.02 & \textbf{88.97} & \underline{43.42} & 64.26 & 89.71 & 82.37 \\
\text{Learnable} & 100 & 70.37 & 89.91 & 70.41 & \textbf{67.00} & 61.36 & \textbf{65.15} & \textbf{51.58} & 88.81 & 42.97 & \underline{65.08} & 89.57 & 82.28 \\
\bottomrule
\end{tabular}%
}
\end{center}
\vspace{-11pt}
\end{table*}

\begin{table}[t]
\vspace{-7pt}
\caption{Ablation study on the target pool (Pool) and the momentum network (EMA) for avoiding collapse while improving representations learned via supervised ANCL on CIFAR100.}
\label{table:ablation_collapse}
\begin{center}
\begin{tabular}{cc|cc}
\toprule
Pool & EMA  & Collapse & $k$-NN \\
\midrule
\xmark & \xmark & \cmark & 1.00 \\
\cmark & \xmark & \xmark & 73.92 \\
\xmark & \cmark & \xmark & 73.32 \\
\cmark & \cmark & \xmark & 74.55 \\
\bottomrule
\end{tabular}
\end{center}
\vspace{-20pt}
\end{table}

\subsection{Ablation Study on Target Pool Design}

In this section, we investigate the design choices for the target pool.
In our experiments, the pretraining dataset ImageNet-100 consists of 100 classes, such that the probability of missing any class in the target pool is negligible with a target pool size of 8192.
However, with a larger number of classes, some classes might not exist in the target pool if it is updated in a class-agnostic manner.
To address this concern, we consider two alternative target pool designs:
1) managing class-wise queues as the target pool, and
2) maintaining learnable class prototypes using the EMA update rule.
Additionally, we adjust the size of the class-wise queues to determine the optimal number of latent features required to ensure good performance.

As shown in \cref{table:pool-design}, performance remains consistent regardless of the target pool design.
For the class-wise queues, increasing the number of features stored per class slightly enhances performance, with the best performance observed at 20 features per class, though the gain is overall marginal.
In  all designs, the size of the target pool grows proportionally to the number of classes and/or the feature dimension, which is equivalent to a linear classifier, such that its memory consumption is negligible; \eg, the linear classifier takes only 2\% of the parameters in ResNet-50.
Nonetheless, a more sophisticated design of the target pool might be effective, which we leave for future works.

\subsection{Ablation Study on Representation Collapse}

In this section, we investigate when collapse occurs in supervised ANCL.
Specifically, we investigate the effect of the target pool and the momentum network, where the method only with the target pool is essentially \textsc{SupSiam}, and the one with both components corresponds to \textsc{SupBYOL}.
We pretrain ResNet-18 followed by a 2-layer MLP projector and predictor on CIFAR100.

As observed in \cref{table:ablation_collapse}, employing either the target pool or the momentum network effectively prevents collapse.
We hypothesize that updating the target differently from the anchor helps to prevent collapse, which is the common behavior of both strategies.

\section{Conclusion}

In this paper, we study supervised asymmetric non-contrastive learning (ANCL) for representation learning.
We demonstrate that introducing supervision to ANCL reduces intra-class variance, and that balancing the contributions of the supervised and self-supervised losses is crucial to learn good representations.
We experiment the proposed supervised ANCL methods with baselines across various datasets and tasks, demonstrating the effectiveness of supervised ANCL.
We believe our work motivates future research to integrate supervised ANCL into their applications.

\section*{Acknowledgements}

This work was supported by the National Research Foundation of Korea (NRF) grant funded by the Korea government (MSIT) (2022R1A4A1033384) and the Yonsei University Research Fund (2024-22-0148). We thank Jy-yong Sohn and Chungpa Lee for helpful discussions.

\section*{Impact Statement}

(Non-)contrastive learning typically requires substantial training costs; for instance, training ResNet-50 with MoCo-v2 \citep{chen20improved} for 800 epochs requires 9 days on 8 V100 GPUs, raising concerns about environmental impacts, such as carbon emissions.
However, the proposed idea of incorporating supervision leads to learning better representations while maintaining similar computational complexity comparable to that of self-supervised learning.
This suggests that supervision can mitigate computational demands and potentially address associated environmental concerns.

\bibliography{ref}

\begin{thebibliography}{63}
\providecommand{\natexlab}[1]{#1}
\providecommand{\url}[1]{\texttt{#1}}
\expandafter\ifx\csname urlstyle\endcsname\relax
  \providecommand{\doi}[1]{doi: #1}\else
  \providecommand{\doi}{doi: \begingroup \urlstyle{rm}\Url}\fi

\bibitem[Asadi et~al.(2022)Asadi, Mudur, and Belilovsky]{asadi22supbyol}
Asadi, N., Mudur, S., and Belilovsky, E.
\newblock Tackling online one-class incremental learning by removing negative contrasts.
\newblock \emph{arXiv preprint arXiv:2203.13307}, 2022.

\bibitem[Bardes et~al.(2022)Bardes, Ponce, and LeCun]{bardes22vicreg}
Bardes, A., Ponce, J., and LeCun, Y.
\newblock Vicreg: Variance-invariance-covariance regularization for self-supervised learning.
\newblock In \emph{ICLR}, 2022.

\bibitem[Bossard et~al.(2014)Bossard, Guillaumin, and Gool]{bossard14food}
Bossard, L., Guillaumin, M., and Gool, L.~V.
\newblock Food-101 – mining discriminative components with random forests.
\newblock In \emph{ECCV}, 2014.

\bibitem[Cha et~al.(2021)Cha, Lee, and Shin]{cha2021co2l}
Cha, H., Lee, J., and Shin, J.
\newblock Co2l: Contrastive continual learning.
\newblock In \emph{ICCV}, 2021.

\bibitem[Chen et~al.(2022)Chen, Fu, Narayan, Zhang, Song, Fatahalian, and Ré]{chen22perfectly}
Chen, M.~F., Fu, D.~Y., Narayan, A., Zhang, M., Song, Z., Fatahalian, K., and Ré, C.
\newblock Perfectly balanced: Improving transfer and robustness of supervised contrastive learning.
\newblock In \emph{ICML}, 2022.

\bibitem[Chen et~al.(2020{\natexlab{a}})Chen, Kornblith, Norouzi, and Hinton]{chen20simple}
Chen, T., Kornblith, S., Norouzi, M., and Hinton, G.
\newblock A simple framework for contrastive learning of visual representations.
\newblock In \emph{ICML}, 2020{\natexlab{a}}.

\bibitem[Chen \& He(2021)Chen and He]{chen21siamese}
Chen, X. and He, K.
\newblock Exploring simple siamese representation learning.
\newblock In \emph{CVPR}, 2021.

\bibitem[Chen et~al.(2020{\natexlab{b}})Chen, Fan, Girshick, and He]{chen20improved}
Chen, X., Fan, H., Girshick, R., and He, K.
\newblock Improved baselines with momentum contrastive learning.
\newblock \emph{arXiv preprint arXiv:2101.11058}, 2020{\natexlab{b}}.

\bibitem[Chen et~al.(2021)Chen, Xie, and He]{chen21moco-v3}
Chen, X., Xie, S., and He, K.
\newblock An empirical study of training self-supervised vision transformers.
\newblock In \emph{ICCV}, 2021.

\bibitem[Cimpoi et~al.(2014)Cimpoi, Maji, Kokkinos, Mohamed, and Vedaldi]{cimpoi14dtd}
Cimpoi, M., Maji, S., Kokkinos, I., Mohamed, S., and Vedaldi, A.
\newblock Describing textures in the wild.
\newblock In \emph{CVPR}, 2014.

\bibitem[Deng et~al.(2009)Deng, Dong, Socher, Li, Li, and Fei-Fei]{deng2009imagenet}
Deng, J., Dong, W., Socher, R., Li, L.-J., Li, K., and Fei-Fei, L.
\newblock Imagenet: A large-scale hierarchical image database.
\newblock In \emph{CVPR}, 2009.

\bibitem[Dosovitskiy et~al.(2021)Dosovitskiy, Beyer, Kolesnikov, Weissenborn, Zhai, Unterthiner, Dehghani, Minderer, Heigold, Gelly, Uszkoreit, and Houlsby]{dosovitskiy21transformer}
Dosovitskiy, A., Beyer, L., Kolesnikov, A., Weissenborn, D., Zhai, X., Unterthiner, T., Dehghani, M., Minderer, M., Heigold, G., Gelly, S., Uszkoreit, J., and Houlsby, N.
\newblock An image is worth 16x16 words: Transformers for image recognition at scale.
\newblock In \emph{ICLR}, 2021.

\bibitem[Everingham et~al.(2010)Everingham, Gool, Williams, Winn, and Zisserman]{everingham10voc}
Everingham, M., Gool, L.~V., Williams, C. K.~I., Winn, J., and Zisserman, A.
\newblock The pascal visual object classes (voc) challenge.
\newblock \emph{IJCV}, 2010.

\bibitem[Fei-Fei et~al.(2004)Fei-Fei, Fergus, and Perona]{fei04caltech}
Fei-Fei, L., Fergus, R., and Perona, P.
\newblock Learning generative visual models from few training examples: An incremental bayesian approach tested on 101 object categories.
\newblock In \emph{CVPR Workshop}, 2004.

\bibitem[Goyal et~al.(2017)Goyal, Doll{\'a}r, Girshick, Noordhuis, Wesolowski, Kyrola, Tulloch, Jia, and He]{goyal17accurate}
Goyal, P., Doll{\'a}r, P., Girshick, R., Noordhuis, P., Wesolowski, L., Kyrola, A., Tulloch, A., Jia, Y., and He, K.
\newblock Accurate, large minibatch {SGD}: training imagenet in 1 hour.
\newblock \emph{arXiv preprint arXiv:1706.02677}, 2017.

\bibitem[Graf et~al.(2021)Graf, Hofer, Niethammer, and Kwitt]{graf21dissect}
Graf, F., Hofer, C., Niethammer, M., and Kwitt, R.
\newblock Dissecting supervised constrastive learning.
\newblock In \emph{ICML}, 2021.

\bibitem[Grill et~al.(2020)Grill, Strub, Altch{\'e}, Tallec, Richemond, Buchatskaya, Doersch, Avila~Pires, Guo, Gheshlaghi~Azar, et~al.]{grill20bootstrap}
Grill, J.-B., Strub, F., Altch{\'e}, F., Tallec, C., Richemond, P., Buchatskaya, E., Doersch, C., Avila~Pires, B., Guo, Z., Gheshlaghi~Azar, M., et~al.
\newblock Bootstrap your own latent: A new approach to self-supervised learning.
\newblock In \emph{NeurIPS}, 2020.

\bibitem[Gunel et~al.(2021)Gunel, Du, Conneau, and Stoyanov]{gunel21}
Gunel, B., Du, J., Conneau, A., and Stoyanov, V.
\newblock Supervised contrastive learning for pre-trained language model fine-tuning.
\newblock In \emph{ICLR}, 2021.

\bibitem[Halvagal et~al.(2023)Halvagal, Laborieux, and Zenke]{halvagal23implicit}
Halvagal, M.~S., Laborieux, A., and Zenke, F.
\newblock Implicit variance regularization in non-contrastive ssl.
\newblock In \emph{NeurIPS}, 2023.

\bibitem[He et~al.(2016)He, Zhang, Ren, and Sun]{he2016deep}
He, K., Zhang, X., Ren, S., and Sun, J.
\newblock Deep residual learning for image recognition.
\newblock In \emph{CVPR}, 2016.

\bibitem[He et~al.(2020)He, Fan, Wu, Xie, and Girshick]{he20contrast}
He, K., Fan, H., Wu, Y., Xie, S., and Girshick, R.
\newblock Momentum contrast for unsupervised visual representation learning.
\newblock In \emph{CVPR}, 2020.

\bibitem[He et~al.(2022)He, Chen, Xie, Li, Doll{\'a}r, and Girshick]{he2022masked}
He, K., Chen, X., Xie, S., Li, Y., Doll{\'a}r, P., and Girshick, R.
\newblock Masked autoencoders are scalable vision learners.
\newblock In \emph{CVPR}, 2022.

\bibitem[Houben et~al.(2013)Houben, Stallkamp, Salmen, Schlipsing, and Igel]{houben13traffic}
Houben, S., Stallkamp, J., Salmen, J., Schlipsing, M., and Igel, C.
\newblock Detection of traffic signs in real-world images: The german traffic sign detection benchmark.
\newblock In \emph{International Joint Conference on Neural Networks}, 2013.

\bibitem[Ioffe \& Szegedy(2015)Ioffe and Szegedy]{ioffe2015batch}
Ioffe, S. and Szegedy, C.
\newblock Batch normalization: Accelerating deep network training by reducing internal covariate shift.
\newblock In \emph{ICML}, 2015.

\bibitem[Kang et~al.(2021)Kang, Li, Xie, Yuan, and Feng]{kang21balance}
Kang, B., Li, Y., Xie, S., Yuan, Z., and Feng, J.
\newblock Exploring balanced feature spaces for representation learning.
\newblock In \emph{ICLR}, 2021.

\bibitem[Khosla et~al.(2011)Khosla, Jayadevaprakash, Yao, and Fei-Fei]{khosla11dog}
Khosla, A., Jayadevaprakash, N., Yao, B., and Fei-Fei, L.
\newblock Novel dataset for fine-grained image categorization.
\newblock In \emph{CVPR Workshop}, 2011.

\bibitem[Khosla et~al.(2020)Khosla, Teterwak, Wang, Sarna, Tian, Isola, Maschinot, Liu, and Krishnan]{khosla20supcon}
Khosla, P., Teterwak, P., Wang, C., Sarna, A., Tian, Y., Isola, P., Maschinot, A., Liu, C., and Krishnan, D.
\newblock Supervised contrastive learning.
\newblock In \emph{NeurIPS}, 2020.

\bibitem[Kornblith et~al.(2019)Kornblith, Shlens, and Le]{kornblith19do}
Kornblith, S., Shlens, J., and Le, Q.~V.
\newblock Do better imagenet models transfer better?
\newblock In \emph{CVPR}, 2019.

\bibitem[Krizhevsky \& Hinton(2009)Krizhevsky and Hinton]{krizhevsky2009learning}
Krizhevsky, A. and Hinton, G.
\newblock Learning multiple layers of features from tiny images.
\newblock Technical report, University of Toronto, 2009.

\bibitem[Lake et~al.(2015)Lake, Salakhutdinov, and Tenenbaum]{lake15omni}
Lake, B.~M., Salakhutdinov, R., and Tenenbaum, J.~B.
\newblock Human-level concept learning through probabilistic program induction.
\newblock \emph{Science}, 350\penalty0 (6266):\penalty0 1332--1338, 2015.

\bibitem[Lee et~al.(2021{\natexlab{a}})Lee, Lee, Lee, Lee, and Shin]{lee21augself}
Lee, H., Lee, K., Lee, K., Lee, H., and Shin, J.
\newblock Improving transferability of representations via augmentation-aware self-supervision.
\newblock In \emph{NeurIPS}, 2021{\natexlab{a}}.

\bibitem[Lee \& Kim(2015)Lee and Kim]{lee2015equivalence}
Lee, K. and Kim, J.
\newblock On the equivalence of linear discriminant analysis and least squares.
\newblock In \emph{AAAI}, 2015.

\bibitem[Lee et~al.(2021{\natexlab{b}})Lee, Zhu, Sohn, Li, Shin, and Lee]{lee2020mix}
Lee, K., Zhu, Y., Sohn, K., Li, C.-L., Shin, J., and Lee, H.
\newblock i-mix: A domain-agnostic strategy for contrastive representation learning.
\newblock In \emph{ICLR}, 2021{\natexlab{b}}.

\bibitem[Lin et~al.(2014)Lin, Maire, Belongie, Bourdev, Girshick, Hays, Perona, Ramanan, Zitnick, and Dollár]{lin14coco}
Lin, T.-Y., Maire, M., Belongie, S., Bourdev, L., Girshick, R., Hays, J., Perona, P., Ramanan, D., Zitnick, C.~L., and Dollár, P.
\newblock Microsoft coco: Common objects in context.
\newblock In \emph{ECCV}, 2014.

\bibitem[Liu \& Nocedal(1989)Liu and Nocedal]{liu1989limited}
Liu, D.~C. and Nocedal, J.
\newblock On the limited memory bfgs method for large scale optimization.
\newblock \emph{Mathematical programming}, 45\penalty0 (1-3):\penalty0 503--528, 1989.

\bibitem[Liu et~al.(2022)Liu, Suganuma, and Okatani]{liu22gap}
Liu, K.-J., Suganuma, M., and Okatani, T.
\newblock Bridging the gap from asymmetry tricks to decorrelation principles in non-contrastive self-supervised learning.
\newblock In \emph{NeurIPS}, 2022.

\bibitem[Loshchilov \& Hutter(2017)Loshchilov and Hutter]{loshchilov17sgdr}
Loshchilov, I. and Hutter, F.
\newblock Sgdr: Stochastic gradient descent with warm restarts.
\newblock In \emph{ICLR}, 2017.

\bibitem[Loshchilov \& Hutter(2019)Loshchilov and Hutter]{loshcilov19adamw}
Loshchilov, I. and Hutter, F.
\newblock Decoupled weight decay regularization.
\newblock In \emph{ICLR}, 2019.

\bibitem[Maaten \& Hinton(2008)Maaten and Hinton]{maaten2008visualizing}
Maaten, L. v.~d. and Hinton, G.
\newblock Visualizing data using t-sne.
\newblock \emph{JMLR}, 9\penalty0 (Nov), 2008.

\bibitem[Maji et~al.(2013)Maji, Rahtu, Kannala, Blaschko, and Vedaldi]{maji13aircraft}
Maji, S., Rahtu, E., Kannala, J., Blaschko, M., and Vedaldi, A.
\newblock Fine-grained visual classification of aircraft.
\newblock \emph{arXiv preprint arXiv:1306.5151}, 2013.

\bibitem[Majumder et~al.(2021)Majumder, Ravichandran, Maji, Achille, Polito, and Soatto]{majumder21supmoco}
Majumder, O., Ravichandran, A., Maji, S., Achille, A., Polito, M., and Soatto, S.
\newblock Supervised momentum contrastive learning for few-shot classification.
\newblock \emph{arXiv preprint arXiv:2101.11058}, 2021.

\bibitem[Maser et~al.(2023)Maser, Park, Lin, Lee, Frey, and Watkins]{maser23supsiam}
Maser, M., Park, J.~W., Lin, J. Y.-Y., Lee, J.~H., Frey, N.~C., and Watkins, A.
\newblock Supsiam: Non-contrastive auxiliary loss for learning from molecular conformers.
\newblock \emph{arXiv preprint arXiv:2302.07754}, 2023.

\bibitem[Nilsback \& Zisserman(2008)Nilsback and Zisserman]{nilsback08flower}
Nilsback, M.-E. and Zisserman, A.
\newblock Automated flower classification over a large number of classes.
\newblock In \emph{Proceedings of the Indian Conference of Computer Visions, Graphics and Image Processing}, 2008.

\bibitem[Oreshkin et~al.(2018)Oreshkin, Rodriguez, and Lacoste]{oreshkin18fc}
Oreshkin, B.~N., Rodriguez, P., and Lacoste, A.
\newblock Tadam: Task dependent adaptive metric for improved few-shot learning.
\newblock In \emph{NeurIPS}, 2018.

\bibitem[Papyan et~al.(2020)Papyan, Han, and Donoho]{papyan15collapse}
Papyan, V., Han, X., and Donoho, D.~L.
\newblock Prevalence of neural collapse during the terminal phase of deep learning training.
\newblock \emph{Proceedings of the National Academy of Sciences}, 117\penalty0 (40):\penalty0 24652--24663, 2020.

\bibitem[Parkhi et~al.(2012)Parkhi, Vedaldi, Zisserman, and Jawahar]{parkhi12pet}
Parkhi, O.~M., Vedaldi, A., Zisserman, A., and Jawahar, C.
\newblock Cats and dogs.
\newblock In \emph{CVPR}, 2012.

\bibitem[Penrose(1955)]{penrose1955generalized}
Penrose, R.
\newblock A generalized inverse for matrices.
\newblock In \emph{Mathematical proceedings of the Cambridge philosophical society}, 1955.

\bibitem[Quattoni \& Torralba(2009)Quattoni and Torralba]{quattoni09mit}
Quattoni, A. and Torralba, A.
\newblock Recognizing indoor scenes.
\newblock In \emph{CVPR}, 2009.

\bibitem[Razavian et~al.(2014)Razavian, Azizpour, Sullivan, and Carlsson]{sharif2014cnn}
Razavian, A.~S., Azizpour, H., Sullivan, J., and Carlsson, S.
\newblock Cnn features off-the-shelf: an astounding baseline for recognition.
\newblock In \emph{CVPR DeepVision Workshop}, 2014.

\bibitem[Ren et~al.(2015)Ren, He, Girshick, and Sun]{ren15rcnn}
Ren, S., He, K., Girshick, R., and Sun, J.
\newblock Faster r-cnn: Towards real-time object detection with region proposal networks.
\newblock In \emph{NeurIPS}, 2015.

\bibitem[Richemond et~al.(2023)Richemond, Tam, Tang, Strub, Piot, and Hill]{richemond2023edge}
Richemond, P.~H., Tam, A., Tang, Y., Strub, F., Piot, B., and Hill, F.
\newblock The edge of orthogonality: a simple view of what makes byol tick.
\newblock In \emph{ICML}, 2023.

\bibitem[Schroeder \& Cui(2018)Schroeder and Cui]{schroeder18fungi}
Schroeder, B. and Cui, Y.
\newblock Fgvcx fungi classification challenge 2018.
\newblock \url{github.com/visipedia/fgvcx_fungi_comp}, 2018.

\bibitem[Tian et~al.(2020)Tian, Krishnan, and Isola]{tian19cmc}
Tian, Y., Krishnan, D., and Isola, P.
\newblock Contrastive multiview coding.
\newblock In \emph{ECCV}, 2020.

\bibitem[Tian et~al.(2021)Tian, Chen, and Ganguli]{tian21}
Tian, Y., Chen, X., and Ganguli, S.
\newblock Understanding self-supervised learning dynamics without contrastive pairs.
\newblock In \emph{ICML}, 2021.

\bibitem[van~den Oord et~al.(2018)van~den Oord, Li, and Vinyals]{oord18cpc}
van~den Oord, A., Li, Y., and Vinyals, O.
\newblock Representation learning with contrastive predictive coding.
\newblock \emph{arXiv preprint arXiv:1807.03748}, 2018.

\bibitem[Wang et~al.(2023)Wang, Zheng, Zhu, Zhou, and Lu]{wang23opera}
Wang, C., Zheng, W., Zhu, Z., Zhou, J., and Lu, J.
\newblock Opera: Omni-supervised representation learning with hierarchical supervisions.
\newblock In \emph{ICCV}, 2023.

\bibitem[Wei et~al.(2021)Wei, Xie, He, Chang, Zhang, Zhou, Li, and Tian]{wei21semantic}
Wei, L., Xie, L., He, J., Chang, J., Zhang, X., Zhou, W., Li, H., and Tian, Q.
\newblock Can semantic labels assist self-supervised visual representation learning?
\newblock In \emph{AAAI}, 2021.

\bibitem[Welinder et~al.(2010)Welinder, Branson, Mita, Wah, Schroff, Belongie, and Perona]{welinder2010cub}
Welinder, P., Branson, S., Mita, T., Wah, C., Schroff, F., Belongie, S., and Perona, P.
\newblock {Caltech-UCSD Birds 200}.
\newblock Technical report, California Institute of Technology, 2010.

\bibitem[Wu et~al.(2018)Wu, Xiong, Yu, and Lin]{wu2018unsupervised}
Wu, Z., Xiong, Y., Yu, S.~X., and Lin, D.
\newblock Unsupervised feature learning via non-parametric instance discrimination.
\newblock In \emph{CVPR}, 2018.

\bibitem[Xiao et~al.(2010)Xiao, abd Krista A. Ehinger~abd Aude~Oliva, and Torralba]{xiao10sun}
Xiao, J., abd Krista A. Ehinger~abd Aude~Oliva, J.~H., and Torralba, A.
\newblock Sun database: Large-scale scene recognition from abbey to zoo.
\newblock In \emph{CVPR}, 2010.

\bibitem[Xue et~al.(2023)Xue, Joshi, Gan, Chen, and Mirzasoleiman]{xue23learnt}
Xue, Y., Joshi, S., Gan, E., Chen, P.-Y., and Mirzasoleiman, B.
\newblock Which features are learnt by contrastive learning? on the role of simplicity bias in class collapse and feature suppression.
\newblock In \emph{ICML}, 2023.

\bibitem[Zbontar et~al.(2021)Zbontar, Jing, Misra, LeCun, and Deny]{zbontar21barlow}
Zbontar, J., Jing, L., Misra, I., LeCun, Y., and Deny, S.
\newblock Barlow twins: Self-supervised learning via redundancy reduction.
\newblock In \emph{ICML}, 2021.

\bibitem[Zhuo et~al.(2023)Zhuo, Wang, Ma, and Wang]{zhuo23rdm}
Zhuo, Z., Wang, Y., Ma, J., and Wang, Y.
\newblock Towards a unified theoretical understanding of non-contrastive learning via rank differential mechanism.
\newblock In \emph{ICLR}, 2023.

\end{thebibliography}
\bibliographystyle{icml2024}

\newpage
\appendix
\onecolumn
\numberwithin{table}{section}
\numberwithin{figure}{section}
\numberwithin{equation}{section}

\section{Detailed Proofs for \Cref{sec:theoretical}}
\label{sec:proof}

\subsection{Derivation of 
\textbf{\Cref{eq:loss_expect}}}
\vspace{-5pt}

To derive this, recall the supervised ANCL loss with constraints in \cref{eq:loss_contrained}:
\begin{align*}
    \ell
 &= \alpha \left\| W_p z_1 - z_2 \right\|_2^2 +
    (1-\alpha) \left\| W_p z_1 - z_2^\text{sup} \right\|_2^2
    \quad\text{ s.t. }\quad
    \left\| z_2 \right\|_2^2
  = \left\| z_2^\text{sup} \right\|_2^2
  = \left\| W_p z_1 \right\|_2^2
  = 1. \tag{\ref{eq:loss_contrained}}
\end{align*}
We first expand the loss in \cref{eq:loss_contrained} and apply constraints to simplify the expression:
\begin{align}
\begin{aligned}[b]
    \ell
 &= \alpha \left(
        \left\| W_p z_1 \right\|_2^2 +
        \left\| z_2 \right\|_2^2 -
        2 z_1^\top W_p^\top z_2
    \right) +
    (1-\alpha) \left(
        \left\| W_p z_1 \right\|_2^2 +
        \left\| z_2^\text{sup} \right\|_2^2 -
        2 z_1^\top W_p^\top z_2^\text{sup}
    \right) \\
 &= \alpha \left( 2 - 2 z_1^\top W_p^\top z_2 \right) +
    (1-\alpha) \left( 2 - 2 z_1^\top W_p^\top z_2^\text{sup} \right) \\
 &= 2 - 2 \alpha \cdot z_1^\top W_p^\top z_2 +
    2 (1-\alpha) \cdot z_1^\top W_p^\top z_2^\text{sup}.
\end{aligned}
\end{align}
Then, the Lagrangian function is formulated as follows:
{%
\thinmuskip=2.5mu 
\medmuskip=3.5mu plus 2mu minus 4mu 
\thickmuskip=5mu plus 5mu
\begin{align}
\begin{aligned}[b]
    \mathcal{L}
 &= 2 -2 \alpha \cdot z_1^\top W_p^\top z_2
    -2 (1-\alpha) \cdot z_1^\top W_p^\top z_2^\text{sup} + \lambda_1 \left(
        z_2^\top z_2 - 1
    \right) +
    \lambda_2 \left(
        z_2^{\text{sup} \top} z_2^\text{sup} - 1
    \right) +
    \lambda_3 \left( 
        z_1^\top W_p^\top W_p z_1 - 1
    \right) \\
 &= 2 -2 \alpha \cdot \tr\left( W_p^\top z_2 z_1^\top \right)
    -2 (1-\alpha) \cdot \tr\left( W_p^\top z_2^\text{sup} z_1^\top \right) \\
   &\quad + \lambda_1 \left(
        \tr\left( z_2 z_2^\top \right) - 1
    \right) +
    \lambda_2 \left(
        \tr\left( z_2^\text{sup} z_2^{\text{sup} \top} \right) - 1
    \right) +
    \lambda_3 \left( 
        \tr\left( W_p^\top W_p z_1 z_1^\top \right) - 1
    \right),
\end{aligned}
\end{align}
}%
where $\lambda_1$, $\lambda_2$ and $\lambda_3$ are the Lagrange multipliers.
Finally, taking the expectation over $x_1$, $x_2$, and $x_2^\text{sup}$ yields \cref{eq:loss_expect}:
\begin{align*}
    \mathcal{L}
 &= 2 -2 \alpha \cdot \tr\left( W_p^\top \mathbb{E}\left[ z_2 z_1^\top \right] \right)
    -2 (1-\alpha) \cdot \tr\left( W_p^\top \mathbb{E}\left[ z_2^\text{sup} z_1^\top \right] \right) \\
 &\quad + \lambda_1 \left(
        \tr\left( \mathbb{E}\left[ z_2 z_2^\top \right] \right) - 1
    \right) +
    \lambda_2 \left(
        \tr\left( \mathbb{E}\left[ z_2^\text{sup} z_2^{\text{sup} \top} \right] \right) - 1
    \right) +
    \lambda_3 \left( 
        \tr\left( W_p^\top W_p \mathbb{E}\left[ z_1 z_1^\top \right] \right) - 1
    \right). \tag{\ref{eq:loss_expect}}
\end{align*}

\subsection{Proof of \textbf{\Cref{proposition:cov}}}
\vspace{-5pt}

\propositionCov*
\vspace{-5pt}

\begin{proof}
\vspace{-5pt}
Let $S_B = \frac{1}{C} \sum_y \mu_y \mu_y^{\top}$ be the inter-class covariance,
$S_W = \frac{1}{C} \sum_y \Sigma_y$ be the intra-class covariance, and
$S_e = \sigma^2_e I$ be the variance of the augmentation noise.
\begin{align*}
    \mathbb{E}_{x_1}\left[ z_1 z_1^{\top} \right] 
 &= W \mathbb{E}_X \left[
        \mathbb{E}_{\widetilde{X}|X} \left[ x_1 x_1^\top \right]
    \right] W^\top
  = W \mathbb{E}_X \left[
        x x^\top + \sigma^2_e I
    \right] W^\top \nonumber \\
 &= W \mathbb{E}_Y \left[
        \mathbb{E}_{X|Y} \left[ xx^\top + \sigma^2_e I \right]
    \right] W^\top
  = W \mathbb{E}_Y \left[
        \mu_y \mu_y^\top + \Sigma_y + \sigma^2_e I
    \right] W^\top \nonumber \\
 &= W \left(
        \frac{1}{C} \sum_y \left( \mu_y \mu_y^\top + \Sigma_y \right) + \sigma^2_e I
    \right) W^\top \nonumber \\
 &= W \left( S_B + S_W + S_e \right)W^\top
  = (1 + \sigma_e^2) W W^\top, \\
    \mathbb{E}_{x_1,x_2}\left[ z_2 z_1^{\top} \right] 
 &= W \mathbb{E}_X \left[
        \mathbb{E}_{\widetilde{X}|X} \left[x_2 x_1^\top \right]
    \right] W^\top
  = W \mathbb{E}_X \left[
        \mathbb{E}_{\widetilde{X}|X} \left[ x_2 \right] \mathbb{E}_{\widetilde{X}|X} \left[ x_1^\top \right]
    \right] W^\top
  = W \mathbb{E}_X \left[ xx^\top \right] W^\top \nonumber \\
 &= W \mathbb{E}_Y \left[
        \mathbb{E}_{X|Y} \left[xx^\top \right]
    \right] W^\top
  = W \mathbb{E}_Y \left[
        \mu_y \mu_y^\top + \Sigma_y
    \right] W^\top \nonumber \\
 &= W \left(
        \frac{1}{C} \sum_y \left(\mu_y \mu_y^\top + \Sigma_y \right)
    \right) W^\top \nonumber \\
 &= W \left( S_B + S_W \right)W^\top
  = W W^\top, \\
    \mathbb{E}_{x_1,x_2^\text{sup}}\left[ z_2^\text{sup} z_1^\top \right] 
 &= W \mathbb{E}_{X}\left[
        \mathbb{E}_{\widetilde{X}|X} \left[x_2^\text{sup} x_1^\top \right]
    \right] W^\top 
    = W \mathbb{E}_{X}\left[
        \mathbb{E}_{\widetilde{X}|X} \left[x_2^\text{sup} \right] \mathbb{E}_{\widetilde{X}|X} \left[x_1^\top \right]
    \right] W^\top
    = W \mathbb{E}_{X} \left[ x^\text{sup} x^\top \right] W^\top \nonumber \\
 &= W \mathbb{E}_Y \left[
        \mathbb{E}_{X | Y} \left[x^\text{sup} x^\top \right]
    \right] W^\top 
    = W \mathbb{E}_Y \left[
        \mathbb{E}_{X | Y} \left[x^\text{sup} \right] \mathbb{E}_{X | Y} \left[x^\top \right]
    \right] W^\top
  = W \mathbb{E}_Y \left[ \mu_y \mu_y^\top \right] W^\top  \nonumber \\
 &= W \left(
        \frac{1}{C} \sum_y \mu_y \mu_y^\top
    \right) W^\top = W S_B W^\top.
\tag{\ref{eq:cov}}
\end{align*}
Let $S_B = V \Lambda_B V^\top$ be the eigendecomposition, where $V$ is an orthogonal matrix and $\Lambda_B$ is a diagonal matrix of the eigenvalues.
Then, $S_T = S_B + S_W$ and $S_e$ share the same eigenspace with $S_B$, as they are (scaled) identity matrices.
\begin{align*}
    \mathbb{E}\left[ z_1 z_1^{\top} \right]
 &= W \left( S_B + S_W + S_e \right)W^\top
  = W V \left( \Lambda_B + \Lambda_W + \sigma^2_e I \right) V^\top W^\top, \\
    \mathbb{E}\left[ z_2 z_1^{\top} \right]
 &= W \left( S_B + S_W \right)W^\top
  = W V \left( \Lambda_B + \Lambda_W \right) V^\top W^\top, \\
    \mathbb{E}\left[ z_2^\text{sup} z_1^\top \right]
 &= W S_B W^\top
  = W V \Lambda_B V^\top W^\top,
\tag{\ref{eq:cov_eig}}
\end{align*}
where $\Lambda_W = I - \Lambda_B$ is the eigenvalue matrix of $S_W$.
It can be seen that the covariance matrices of features in \cref{eq:cov_eig} share the same eigenspace in the data space.
\end{proof}

\subsection{Proof of \textbf{\Cref{thm:optimal_pred}}}

\theoremOptimalPred*

\begin{proof}
Recall \cref{eq:loss_expect}:
\begin{align*}
    \mathcal{L}
 &= 2 -2 \alpha \cdot \tr\left( W_p^\top \mathbb{E}\left[ z_2 z_1^\top \right] \right)
    -2 (1-\alpha) \cdot \tr\left( W_p^\top \mathbb{E}\left[ z_2^\text{sup} z_1^\top \right] \right) \\
 &\quad + \lambda_1 \left(
        \tr\left( \mathbb{E}\left[ z_2 z_2^\top \right] \right) - 1
    \right) +
    \lambda_2 \left(
        \tr\left( \mathbb{E}\left[ z_2^\text{sup} z_2^{\text{sup} \top} \right] \right) - 1
    \right) +
    \lambda_3 \left( 
        \tr\left( W_p^\top W_p \mathbb{E}\left[ z_1 z_1^\top \right] \right) - 1
    \right). \tag{\ref{eq:loss_expect}}
\end{align*}
To derive the optimal $W_p^*$, we take the partial derivative $\frac{\partial \mathcal{L}}{\partial W_p}$ and replace the expression of covariance matrices with \cref{eq:cov}:
\begin{align}
\begin{aligned}[b]
\label{eq:dldwp}
    \frac{\partial \mathcal{L}}{\partial W_p}
 &= -2 \alpha \cdot \mathbb{E}\left[ z_2 z_1^\top \right]
    -2 (1-\alpha) \cdot \mathbb{E}\left[ z_2^\text{sup} z_1^\top \right] +
    2\lambda_3 W_p \mathbb{E}\left[ z_1 z_1^\top \right] \\
 &= -2 \alpha \cdot W W^\top
    -2 (1-\alpha) \cdot W S_B W^\top +
    2\lambda_3 \left( 1 + \sigma^2_e \right) \cdot W_p W W^\top.
\end{aligned}
\end{align}
By setting $\frac{\partial \mathcal{L}}{\partial W_p} = 0$, we obtain the optimal predictor $W_p^*$:
\begin{align}
\begin{aligned}[b]
    \lambda_3 \left( 1 + \sigma^2_e \right) \cdot W_p^* W W^\top
 &= W \left( \alpha I + (1-\alpha) S_B \right) W^\top \\
 &= W \left( S_B + \alpha S_W \right) W^\top. \\
    \therefore W_p^* 
 &= \frac{1}{\lambda_3 \left( 1 + \sigma^2_e I \right)}
    W \left( S_B + \alpha S_W \right) W^+.
\end{aligned}
\end{align}
Finally, by substituting the covariance matrices with the eigendecomposition as in \cref{eq:cov}, we obtain the following expression:
\begin{align}
    W_p^*= \frac{1}{\lambda_3} WV \left( \Lambda_B + \alpha \Lambda_W \right) \left( \Lambda_B + \Lambda_W + \sigma^2_e I \right)^{-1} V^\top W^+.
\end{align}
From this expression, the optimal predictor $W_p^*$ can be interpreted through a sequence of hypothetical transformations:
1)~mapping features to the data space,
2)~eliminating the augmentation noise and reducing the intra-class variance by a factor of $\alpha$, and
3)~mapping back to the feature space.
\end{proof}
It is noteworthy that \citet{zhuo23rdm} derived an optimal predictor similar to \cref{thm:optimal_pred}.
However, their focus was on the elimination of augmentation noise in the feature space in the context of self-supervised learning.

\newpage
\subsection{Proof of \textbf{\Cref{thm:align}}}

\theoremAlign*

\begin{proof}
Recall \cref{eq:loss_expect}:
\begin{align*}
    \mathcal{L}
 &= 2 -2 \alpha \cdot \tr\left( W_p^\top \mathbb{E}\left[ z_2 z_1^\top \right] \right)
    -2 (1-\alpha) \cdot \tr\left( W_p^\top \mathbb{E}\left[ z_2^\text{sup} z_1^\top \right] \right) \\
 &\quad + \lambda_1 \left(
        \tr\left( \mathbb{E}\left[ z_2 z_2^\top \right] \right) - 1
    \right) +
    \lambda_2 \left(
        \tr\left( \mathbb{E}\left[ z_2^\text{sup} z_2^{\text{sup} \top} \right] \right) - 1
    \right) +
    \lambda_3 \left( 
        \tr\left( W_p^\top W_p \mathbb{E}\left[ z_1 z_1^\top \right] \right) - 1
    \right),
\tag{\ref{eq:loss_expect}}
\end{align*}
where stop-gradient is applied to $z_2$ and $z_2^\text{sup}$.
Recall the partial derivative $\frac{\partial \mathcal{L}}{\partial W_p}$ is derived in \cref{eq:dldwp}:
\begin{align*}
    \frac{\partial \mathcal{L}}{\partial W_p}
 &= -2 \alpha \cdot \mathbb{E}\left[ z_2 z_1^\top \right]
    -2 (1-\alpha) \cdot \mathbb{E}\left[ z_2^\text{sup} z_1^\top \right] +
    2\lambda_3 W_p \mathbb{E}\left[ z_1 z_1^\top \right] \\
 &= -2 \alpha \cdot W W^\top
    -2 (1-\alpha) \cdot W S_B W^\top +
    2\lambda_3 \left( 1 + \sigma^2_e \right) \cdot W_p W W^\top.
\tag{\ref{eq:dldwp}}
\end{align*}
To derive the partial derivative $\frac{\partial \mathcal{L}}{\partial W}$, we express \cref{eq:loss_expect} in terms of $W$'s and $x$'s:
{%
\begin{align}
\begin{aligned}[b]
\label{eq:loss_expect_x}
    \mathcal{L}
 &= 2 -2 \alpha \cdot \tr\left( W_p^\top \widehat{W} \mathbb{E}\left[ x_2 x_1^\top \right] W^\top \right)
    -2 (1-\alpha) \cdot \tr\left( W_p^\top \widehat{W} \mathbb{E}\left[ x_2^\text{sup} x_1^\top \right] W^\top \right) \\
 &\quad + \lambda_1 \left(
        \tr\left( \widehat{W} \mathbb{E}\left[ x_2 x_2^\top \right] \widehat{W}^\top \right) - 1
    \right) +
    \lambda_2 \left(
        \tr\left( \widehat{W} \mathbb{E}\left[ x_2^\text{sup} x_2^{\text{sup} \top} \right] \widehat{W}^\top \right) - 1
    \right) \\
 &\quad + \lambda_3 \left( 
        \tr\left( W_p^\top W_p W \mathbb{E}\left[ x_1 x_1^\top \right] W^\top \right) - 1
    \right),
\end{aligned}
\end{align}
}%
where $W$'s with stop-gradient are emphasized by $\widehat{W} = \sg(W)$, which are regarded as constants when taking the derivative.
Then, the partial derivative $\frac{\partial \mathcal{L}}{\partial W}$ is derived as follows:
\begin{align}
\begin{aligned}[b]
\label{eq:dldw}
    \frac{\partial \mathcal{L}}{\partial W}
 &= -2 \alpha \cdot W_p^\top W \mathbb{E}\left[ x_2 x_1^\top \right]
    -2 (1-\alpha) \cdot W_p^\top W \mathbb{E}\left[ x_2^\text{sup} x_1^\top \right] + 2\lambda_3 \cdot W_p^\top W_p W \mathbb{E}\left[ x_1 x_1^\top \right] \\
 &= -2 \alpha \cdot W_p^\top W
    -2 (1-\alpha) \cdot W_p^\top W S_B +
    2\lambda_3 \left( 1 + \sigma^2_e \right) \cdot W_p^\top W_p W.
\end{aligned}
\end{align}
Left-multiplying \cref{eq:dldwp} by $W_p^\top$ and right-multiplying \cref{eq:dldw} by $W^\top$ establishes the equality of them:
\begin{align}
\begin{aligned}[b]
    W_p^\top \frac{\partial \mathcal{L}}{\partial W_p} 
 &= -2 \alpha \cdot W_p^\top W W^\top
    -2 (1-\alpha) \cdot W_p^\top W S_B W^\top +
    2\lambda_3 \left( 1 + \sigma^2_e \right) \cdot W_p^\top W_p W W^\top \\
 &= \frac{\partial \mathcal{L}}{\partial W} W^\top. 
\label{eq:dldwp_dldw}
\end{aligned}
\end{align}
Now, we consider the update rule with the current iteration number $t$, the learning rate $\beta$, and the weight decay $\eta$:
\begin{align}
\begin{aligned}[b]
    \frac{d W_p}{dt} = -\beta \frac{\partial \mathcal{L}}{\partial W_p} - \eta W_p, \quad
    \frac{d W}{dt} = -\beta \frac{\partial \mathcal{L}}{\partial W} - \eta W.
\end{aligned}
\end{align}
Substituting this expression into \cref{eq:dldwp_dldw} results in the following equality:
\begin{align}
\begin{aligned}[b]
    W_p^\top \frac{d W_p}{dt} + \eta W_p^\top W_p
  = -\beta W_p^\top \frac{\partial \mathcal{L}}{\partial W_p}
  = -\beta \frac{\partial \mathcal{L}}{\partial W} W^\top
  = \frac{d W}{dt} W^\top + \eta W W^\top.
\end{aligned}
\end{align}
Note that this is a differential equation, where it can be solved by multiplying both side by $e^{2 \eta t}$:
\begin{align}
\begin{aligned}[b]
    \frac{d}{dt} \left( e^{2 \eta t} W_p^\top W_p \right)
  = \frac{d}{dt} \left( e^{2 \eta t} W W^\top \right),
\end{aligned}
\end{align}
then, integrating with respect to $t$ and multiplying by $e^{-2 \eta t}$ yields the solution:
\begin{align}
\begin{aligned}[b]
     W_p^\top W_p = W W^\top + e^{-2 \eta t} c,
\end{aligned}
\end{align}
where $c$ is a constant with respect to $t$.
Finally, the constant becomes negligible as $t \rightarrow \infty$, \ie, at the optimal state, such that we obtain the following expression:
\begin{align}
\begin{aligned}[b]
     W_p^{* \top} W_p^* \approx W^* W^{* \top}.
\end{aligned}
\end{align}
The equality implies that they share the eigenspace.
\end{proof}
Note that \cref{thm:align} holds in self-supervised ANCL as shown in \citet{tian21, liu22gap}, and it remains valid when supervision is incorporated.

\section{Toy Experiment Setup}
\label{sec:toy_setting}

We provide a detailed description of toy experiments in \cref{sec:role}.
We generate a synthetic toy dataset to verify that balancing the contributions of supervision and self-supervision is crucial for the generalization of learned representations. 
The dataset consists of three classes, each following a Gaussian distribution with orthogonal mean vectors and a shared isotropic covariance matrix with a scale of 0.35.
The mean vectors are obtained by taking the left singular vectors of a random matrix sampled from a standard Gaussian distribution.
The synthetic data has 2048 dimensions, and data augmentation is performed by replacing 60\% of the dimensions with the corresponding dimensions of the overall data mean vector. 
The training dataset consists of 3000 samples, with 1000 samples per class, and similarly, the test dataset consists of 1500 samples, with 500 samples per class.

For the supervised ANCL approach, \textsc{SupSiam} is utilized with varying $\alpha$, where the encoder, projector, and predictor each consist of a linear layer without batch normalization \citep{ioffe2015batch}.
The output dimension of the projector/predictor is set to 128.
The model is trained for 200 epochs using the SGD optimizer, with a batch size of 256, learning rate of 0.05, momentum of 0.9, and weight decay of 5e-4.
A cosine learning rate schedule \citep{loshchilov17sgdr} is applied except for the predictor, following the prior work \citep{chen21siamese}.

\section{Pretraining Setup}
\label{sec:pretrain_app}

We provide a detailed description of the pretraining setup.
Each model consists of a ResNet-50 encoder \citep{he2016deep} followed by a 2-layer MLP projector and predictor, except for \textsc{SimSiam} and \textsc{SupSiam}, which utilize a 3-layer MLP projector following the original configuration by \citet{chen21siamese}.
We pretrain models on ImageNet-100 \citep{deng2009imagenet, tian19cmc} for 200 epochs with a batch size of 128.
We utilize the SGD optimizer with a momentum of 0.9, and a weight decay of 1e-4.
A cosine learning rate schedule \citep{loshchilov17sgdr} is applied to the encoder and projector.
We maintain a constant learning rate without decay for the predictor, following the prior work \citep{chen21siamese}.
Other method-specific details are provided below:

\begin{itemize}

    \item \textbf{\textsc{SimCLR}} \citep{chen20simple}. 
    The learning rate is set to 0.1 and the temperature parameter for contrastive loss is 0.1.
    The projector consists of 2 MLP layers with an output dimension of 128.
    \begin{equation}
        \ell_{\text{SimCLR}} = -\text{log}\frac{\text{exp}(z_1 \cdot z_2 / \tau ) }{\sum_{z_a \in z_2 \cup Z_n}\text{exp}(z_1 \cdot z_a / \tau) }, 
    \end{equation}
    where $z_2$ and $z_a$ are L2-normalized and $Z_n$ is the set of negative pairs of $z_1$ obtained from the same batch.
    
    \item \textbf{\textsc{SupCon}} \citep{khosla20supcon}.
    The learning rate is set to 0.15 and the temperature parameter for contrastive loss is 0.1. 
    The projector consists of 2 MLP layers with an output dimension of 128.
    \begin{equation}        
        \ell_{\text{SupCon}} = -\frac{1}{M+1} \text{log}\frac{\text{exp}(z_1 \cdot z_2 / \tau ) }{\sum_{z_a \in B' \cup Z_n}\text{exp}(z_1 \cdot z_a / \tau) } -\frac{1}{M+1}\sum_{z_j \in B' \setminus z_2} \text{log}\frac{\text{exp}(z_1 \cdot z_j / \tau ) }{\sum_{z_a \in B' \cup Z_n}\text{exp}(z_1 \cdot z_a / \tau) }, 
    \end{equation}
    where $z_2$ and $z_a$ are L2-normalized, $B'$ is the set of positive pairs of $z_1$ obtained from the same batch, with a cardinality of $M+1$ and $Z_n$ is the set of negative pairs of $z_1$ obatined from the same batch.
    
    \item \textbf{\textsc{MoCo-v2}} \citep{chen20improved}. 
    The learning rate is set to 0.03 and the temperature parameter for contrastive loss is 0.2. 
    The size of memory bank (target pool) $|Q|$ is 8192, and the exponential moving average (EMA) parameter is 0.999. 
    The projector consists of 2 MLP layers with an output dimension of 128.
    \begin{equation}
        \ell_{\text{MoCo}} = -\text{log}\frac{\text{exp}(z_1 \cdot z_2 / \tau ) }{\sum_{z_a \in z_2 \cup Z_n}\text{exp}(z_1 \cdot z_a / \tau) }, 
    \end{equation}
    where $z_2$ and $z_a$ are L2-normalized and $Z_n$ is the set of negative pairs of $z_1$ obtained from the queue.

    \item \textbf{\textsc{SupMoCo}} \citep{majumder21supmoco}. 
    The learning rate is set to 0.1 and temperature parameter is 0.07. 
    The size of memory bank (target pool) $|Q|$ is 8192 and the EMA parameter is 0.999. 
    The projector consists of 2 MLP layers with an output dimension of 128.
    \begin{equation}        
        \ell_{\text{SupMoCo}} = -\frac{1}{M+1} \text{log}\frac{\text{exp}(z_1 \cdot z_2 / \tau ) }{\sum_{z_a \in Q' \cup Z_n}\text{exp}(z_1 \cdot z_a / \tau) } -\frac{1}{M+1}\sum_{z_j \in Q' \setminus z_2} \text{log}\frac{\text{exp}(z_1 \cdot z_j / \tau ) }{\sum_{z_a \in Q' \cup Z_n}\text{exp}(z_1 \cdot z_a / \tau) }, 
    \end{equation}
    where $z_2$, $z_a$ and $z_j$ are L2-normalized, $Q'$ is the set of positive pairs of $z_1$ obtained from the same batch and the queue, with a cardinality of $M+1$ and $Z_n$ is the set of negative pairs of $z_1$ obtained from the same batch and the queue.

    \item \textbf{\textsc{BYOL}} \citep{grill20bootstrap}. 
    The learning rate is set to 0.2.
    The EMA parameter starts from 0.996 and is increased to one during training.
    The projector consists of 2 MLP layers with an output dimension of 256. 
    The predictor has 2 MLP layers with a hidden dimension of 4096.
    \begin{equation}        
        \ell_{\text{BYOL}} = \left \|p_1 - \sg(z_2) \right \|_2^2, 
    \end{equation}
    where $p_1$ and $z_2$ are L2-normalized and $\sg$ denotes the stop-gradient.
    
    \item \textbf{\textsc{SupBYOL}}. 
    The learning rate is set to 0.2.
    The size of target pool $|Q|$ is 8192 and the supervised target $z_2^\text{sup}$ is obtained by sampling and averaging all positives in the target pool.
    The EMA parameter starts from 0.996 and is increased to one during training. 
    The projector consists of 2 MLP layers with an output dimension of 256, and the predictor has 2 MLP layers with a hidden dimension of 4096.
    \begin{equation}        
        \ell_{\text{SupBYOL}} = \alpha \cdot \left \|p_1 - \sg(z_2) \right \|_2^2 + (1-\alpha) \cdot \left\| p_1 - \sg\left(z_2^\text{sup}\right)\right\|_2^2,
        \ z_2^\text{sup} = \frac{1}{M}\sum_{z'_2 \in Q_y} z'_2,
    \end{equation}
    where $p_1$, $z_2$ and $z'_2$ are L2-normalized, $\sg$ denotes the stop-gradient, and $Q_y \subseteq Q$ is the set of targets of $p_1$ sampled from the target pool sharing the sample label with $p_1$, with a cardinality of $M$.
    
    \item \textbf{\textsc{SimSiam}} \citep{chen20simple}. 
    The learning rate is set to 0.2
    with a linear learning rate warm-up for the first 40 epochs.
    The projector consists of 3 MLP layers with an output dimension of 2048. 
    The predictor has 2 MLP layers with a hidden dimension of 512.
    \begin{equation}        
        \ell_{\text{SimSiam}} = \left \|p_1 - \sg(z_2) \right \|_2^2, 
    \end{equation}
    where $p_1$ and $z_2$ are L2-normalized and $\sg$ denotes the stop-gradient.

    \item \textbf{\textsc{SupSiam}}.
    The learning rate is set to 0.2 with a linear learning rate warm-up for the first 40 epochs.
    The size of target pool $|Q|$ is 8192 and the supervised target $z_2^\text{sup}$ is obtained by sampling and averaging all positives in the target pool.
    The projector consists of 3 MLP layers with an output dimension of 2048, and the predictor has 2 MLP layers with a hidden dimension of 512.
    \begin{equation}        
        \ell_{\text{SupSiam}} = \alpha \cdot \left \|p_1 - \sg(z_2) \right \|_2^2 + (1-\alpha) \cdot \left \|p_1 - \sg(z_2^\text{sup}) \right \|_2^2 , 
        \ z_2^\text{sup} = \frac{1}{M}\sum_{z'_2 \in Q_y} z'_2,
    \end{equation}
    where $p_1$, $z_2$ and $z'_2$ are L2-normalized, $\sg$ denotes the stop-gradient, and $Q_y \subseteq Q$ is the set of targets of $p_1$ sampled from the target pool sharing the sample label with $p_1$, with a cardinality of $M$.
    
\end{itemize}

\section{Datasets}
\label{sec:dataset}

\begin{table}[h]
\caption{Detailed summary of datasets.}
\label{datasets}
\begin{center}
\resizebox{\columnwidth}{!}{%
\begin{tabular}{clrrrrc}
\toprule
Category & Dataset & \# of classes & Train set & Valid set & Test set & Metric \\
\midrule
\multirow{11}{*}{\shortstack[c]{(a) Transfer learning via\\linear evaluation}} & CIFAR10 \citep{krizhevsky2009learning} & 10 & 45000 & 5000 & 10000 & Top-1 accuracy \\
& CIFAR100 \citep{krizhevsky2009learning} & 100 & 45000 & 5000 & 10000 & Top-1 accuracy \\
& DTD (split 1) \citep{cimpoi14dtd} & 47 & 1880 & 1880 & 1880 & Top-1 accuracy \\
& Food \citep{bossard14food} & 101 & 68175 & 7575 & 25250 & Top-1 accuracy \\
& MIT67 \citep{quattoni09mit} & 67 & 4690 & 670 & 1340 & Top-1 accuracy \\
& SUN397 (split 1) \citep{xiao10sun} & 397 & 15880 & 3970 & 19850 & Top-1 accuracy \\
& Caltech101 \citep{fei04caltech} & 101 & 2525 & 505 & 5647 & Mean per-class accuracy \\
& CUB200 \citep{welinder2010cub} & 200 & 4990 & 1000 & 5794 & Mean per-class accuracy \\
& Dogs \citep{khosla11dog, deng2009imagenet} & 120 & 10800 & 1200 & 8580 & Mean per-class accuracy \\
& Flowers \citep{nilsback08flower} & 102 & 1020 & 1020 & 6149 & Mean per-class accuracy \\
& Pets \citep{parkhi12pet} & 37 & 2940 & 740 & 3669 & Mean per-class accuracy \\
\midrule
\multirow{8}{*}{(b) Few-shot classification} & Aircraft \citep{maji13aircraft} & 100 & & & 10000 & Average accuracy \\
& CUB200 \citep{welinder2010cub} & 200 & & & 11745 & Average accuracy \\
& FC100 \citep{oreshkin18fc} & 20 & & & 12000 & Average accuracy \\
& Flowers \citep{nilsback08flower} & 102 & & & 8189 & Average accuracy \\
& Fungi \citep{schroeder18fungi} & 1394 & & & 89760 & Average accuracy \\
& Omniglot \citep{lake15omni} & 1623 & & & 32460 & Average accuracy \\
& DTD \citep{cimpoi14dtd} & 47 & & & 5640 & Average accuracy \\
& Traffic Signs \citep{houben13traffic} & 43 & & & 12630 & Average accuracy \\
\bottomrule
\end{tabular}%
}
\end{center}
\end{table}

\Cref{datasets} provides a comprehensive overview of datasets, including evaluation metrics for both (a) transfer learning via linear evaluation and (b) few-shot classification. 
For datasets without an official validation set, a random split is performed using the entire training set.
For the few-shot task, the complete dataset is utilized for all datasets except FC100 \citep{oreshkin18fc}.
In the case of FC100 \citep{oreshkin18fc}, a meta-test split is used.
Detailed evaluation protocols are outlined in \cref{sec:evaluation}.

\section{Evaluation Protocol}
\label{sec:evaluation}

\subsection{Transfer Learning via Linear Evaluation}

The linear evaluation protocol for transfer learning follows from \citet{kornblith19do} and \citet{lee21augself}.
Specifically, we divide the entire training dataset into a train set and a validation set to tune the regularization parameter by minimizing the L2-regularized cross-entropy loss using L-BFGS \citep{liu1989limited}. 
Train and validation set splits are shown in \cref{datasets}.
With the best parameter, we extract the frozen representations of 224 $\times$ 224 center-cropped images without data augmentation and train the linear classifier with the entire training dataset, including the validation set. 

\subsection{Few-Shot Classfication}

We adhere to the few-shot classification evaluation protocol outlined by \citet{lee21augself}.
Specifically, we conduct logistic regression using the frozen representations extracted from 224 $\times$ 224 images without data augmentation in an $N$-way $K$-shot episode.
It's important to note that as the encoder remains frozen, this protocol does not involve a fine-tuning approach.

\section{Additional Experiments}

We conduct additional experiments with \textsc{SupSiam}, varying the loss parameter $\alpha$, the number of positives, denoted as $M$ and the batch size.
During the experiments on batch size, we also incorporate contrastive learning, specifically \textsc{SupCon}.
Given that the performance recorded in the \cref{table:transfer} might be suboptimal due to pretraining with a batch size of 128, which could be too small, we re-pretrain \textsc{SupCon} using an increased batch size.
Unless specified otherwise, the remaining settings follow the setup outlined in \cref{sec:pretrain_app}.
We apply the evaluation methodology outlined in \cref{sec:evaluation} to the dataset introduced in \cref{sec:dataset}.

\subsection{Transfer Learning with Different $\alpha$}

We conduct experiments with various $\alpha$ values to explore the relationship between intra-class variance reduction and representation quality.
\Cref{table:alpha} presents the linear evaluation performances for different $\alpha$ values. 
The model performs best in most cases when $\alpha$ is set to 0.5.
Interestingly, the optimal $\alpha$ appears to vary depending on the downstream dataset.
Nevertheless, it is crucial to note that $\alpha$ should always fall within the range (0, 1) to effectively capture within-class diversity, thereby proving beneficial for downstream tasks.

\begin{table*}[h]
\caption{Transfer learning via linear evaluation results on various downstream datasets, where the model is \textsc{SupSiam} pretrained with different $\alpha$ on ImageNet-100.
\textit{Avg} represents the average performance across each dataset.
For each dataset, the \textbf{best results} are in bold and the \underline{second-best results} are underlined.
}
\label{table:alpha}
\begin{center}
\resizebox{\textwidth}{!}{%
\begin{tabular}{c|cccccccccccc}
\toprule
$\alpha$ & \textit{Avg} & CIFAR10 & CIFAR100 & DTD & Food & MIT67 & SUN397 & Caltech & CUB200 & Dogs & Flowers & Pets \\
\midrule
0.0 & 69.33 & 89.18 & 69.41 & 65.53 & 60.72 & 65.05 & 50.81 & 88.83 & 41.46 & 61.51 & 90.04 & 80.09 \\
0.2 & 70.14 & \underline{89.89} & \underline{70.56} & 65.89 & 61.03 & \underline{65.25} & 51.34 & \underline{88.85} & 42.07 & \underline{64.28} & 90.12 & \underline{82.27} \\
0.5 & \textbf{70.40} & \textbf{89.95} & \textbf{70.88} & \underline{66.51} & 61.46 & 64.45 & \underline{51.50} & \textbf{88.86} & \textbf{43.48} & \textbf{64.65} & \textbf{90.27} & \textbf{82.38} \\
0.8 & \underline{70.28} & 89.39 & 70.04 & \textbf{67.08} & \textbf{64.06} & \textbf{66.00} & \textbf{51.98} & 87.45 & \underline{42.16} & 62.94 & \underline{90.26} & 81.76 \\
1.0 & 67.07 & 87.28 & 66.41 & 66.06 & \underline{63.44} & 64.68 & 50.69 & 85.00 & 36.10 & 54.57 & 88.38 & 75.13 \\
\bottomrule
\end{tabular}%
}
\end{center}
\vspace{-10pt}
\end{table*}

\subsection{Ablation Study: Number of Positives from Target Pool}
\label{sec:numpos}
We conduct a study on $M$, which represents the number of positive samples from the target pool. 
As shown in \cref{table:numpos}, the model demonstrates robustness to the number of positives.
Even when sampling only one positive from the target pool, the performance is similar to sampling many positives.

\begin{table*}[h]
\caption{Transfer learning via linear evaluation results on various downstream datasets, where the model is \textsc{SupSiam}-pretrained with different $M$ on ImageNet-100.
\textit{all} stands for sampling all positives in the target pool.
\textit{Avg} represents the average performance across each dataset.
For each dataset, the \textbf{best results} are in bold and the \underline{second-best results} are underlined.
}
\label{table:numpos}
\begin{center}
\resizebox{\textwidth}{!}{%
\begin{tabular}{c|cccccccccccc}
\toprule
$M$ & \textit{Avg} & CIFAR10 & CIFAR100 & DTD & Food & MIT67 & SUN397 & Caltech & CUB200 & Dogs & Flowers & Pets \\
\midrule
1 & 70.13 & 89.58 & 70.59 & 65.75 & 61.39 & 64.85 & 51.41 & 88.57 & 42.80 & 64.52 & \underline{89.92} & 82.07 \\
4 & \textbf{70.40} & \underline{89.94} & 70.63 & \textbf{66.59} & \textbf{61.66} & \textbf{65.07} & 51.32 & \underline{88.82} & 42.88 & \underline{64.78} & 89.85 & \textbf{82.88} \\
16 & \underline{70.31} & \underline{89.94} & \underline{70.86} & 65.64 & 61.40 & \underline{64.95} & \textbf{51.54} & 88.65 & \underline{43.17} & \textbf{65.13} & 89.66 & \underline{82.42} \\
\textit{all} & \textbf{70.40} & \textbf{89.95} & \textbf{70.88} & \underline{66.51} & \underline{61.46} & 64.45 & \underline{51.50} & \textbf{88.86} & \textbf{43.48} & 64.65 & \textbf{90.27} & 82.38 \\
\bottomrule
\end{tabular}%
}
\end{center}
\vspace{-10pt}
\end{table*}

\subsection{Transfer Learning with Different Batch Size}
\label{sec:batchsize}

We pretrain \textsc{SupSiam} using an increased batch size 256.
Additionally, we pretrain \textsc{SupCon} with a batch size of 256, as the performance in \cref{table:transfer} pre-trained with a batch size of 128 might be suboptimal.
Moreover, to enhance the diversity of positive and negative samples, we also pretrain \textsc{SupCon} with an additional memory bank (target pool) of size 8192, as described in \citet{khosla20supcon}.
The learning rate scaled linearly \citep{goyal17accurate} with the batch size, \ie, for a batch size of 256, the learning rates are set to 0.3 for \textsc{SupCon} and 0.4 for \textsc{SupSiam}, respectively.

\begin{table*}[h]
\caption{Transfer learning via linear evaluation results on various downstream datasets, where the model is pretrained with different batch size on ImageNet-100.
\textit{Bsz} refers to the batch size during pretraining.
\textit{Avg} represents the average performance across each dataset.
The model marked with $*$ indicates the inclusion of a memory bank.
}
\label{table:bsz}
\begin{center}
\resizebox{\textwidth}{!}{%
\begin{tabular}{c|c|cccccccccccc}
\toprule
\textit{Bsz} & Model & \textit{Avg} & CIFAR10 & CIFAR100 & DTD & Food & MIT67 & SUN397 & Caltech & CUB200 & Dogs & Flowers & Pets \\
\midrule
\multirow{3}{*}{128} & \textsc{SupCon}  & 68.19 & 88.82 & 68.89 & 65.18 & 59.34 & 63.76 & 50.09 & 87.30 & 35.84 & 61.68 & 89.05 & 80.12 \\
 & \textsc{SupCon}$^*$ & 68.97 & 89.70 & 70.48 & 65.65 & 59.06 & 63.43 & 49.86 & 87.97 & 38.76 & 63.74 & 89.22 & 80.83 \\
 & \textsc{SupSiam} & 70.40 & 89.95 & 70.88 & 66.51 & 61.46 & 64.45 & 51.50 & 88.86 & 43.48 & 64.65 & 90.27 & 82.38 \\
\cmidrule(rl){1-14}
 \multirow{3}{*}{256} & \textsc{SupCon}  & 68.42 & 89.10 & 69.40 & 65.32 & 59.21 & 63.25 & 50.63 & 88.22 & 36.05 & 62.60 & 89.10 & 79.80 \\
 & \textsc{SupCon}$^*$ & 68.86 & 89.46 & 70.06 & 65.88 & 58.73 & 63.92 & 50.04 & 87.84 & 38.28 & 63.02 & 89.06 & 81.22 \\
 & \textsc{SupSiam} & 70.44 & 90.07 & 70.77 & 66.28 & 61.89 & 65.18 & 51.77 & 88.83 & 43.09 & 64.90 & 89.99 & 82.11 \\
\bottomrule
\end{tabular}%
}
\end{center}
\end{table*}

As shown in \cref{table:bsz}, supervised ANCL shows a slight improvement in performance when the batch size is increased to 256, though the gain is overall marginal, and it is not heavily influenced by batch size, similar to its self-supervised counterpart \citep{chen21siamese}.
In the case of \textsc{SupCon}, performance improves as the batch size increases, and memory bank provides performance gain, although this gain seems to be slightly reduced as the batch size increases.
However, it still shows lower performance compared to supervised ANCL, which performs well even with a smaller batch size.

\section{ViT Backbone}

To verify the independence of our proposed method from the encoder backbone, we conduct experiments using the ViT backbone \citep{dosovitskiy21transformer}. 
In contrastive learning, \textsc{MoCo-v3} \citep{chen21moco-v3} utilizes ViT as its backbone, and we benchmark this for implementing models such as \textsc{SupMoCo}, \textsc{BYOL}, and \textsc{SupBYOL}.
In \textsc{MoCo-v3}, unlike the previous \textsc{MoCo-v2} \citep{chen20improved}, the queue is removed and a predictor is added, resembling ANCL \citep{grill20bootstrap, chen21siamese}.
For \textsc{SupMoCo} with the ViT backbone, we also incorporate the predictor but retain a queue to ensure the existence of
features sharing the same label with a size of 8192.
Similarly, \textsc{SupBYOL} employs a target pool with a size of 8192.

We pretrain ViT-Small on ImageNet-100 \citep{deng2009imagenet,tian19cmc} for 200 epochs with a batch size of 256. 
Common parameter settings include using the AdamW optimizer \citep{loshcilov19adamw} with a linear learning rate warm-up for the first 40 epochs, a momentum of 0.9, and a weight decay of 0.1.
A cosine learning rate schedule \citep{loshchilov17sgdr} is applied to the encoder and projector.
We maintain a constant learning rate without decay for \textsc{BYOL} and \textsc{SupBYOL} following the prior work \citep{chen21siamese}, while we apply a cosine learning rate schedule to the predictor of \textsc{MoCo-v3} and \textsc{SupMoCo}.

\begin{itemize}
    \item \textbf{\textsc{MoCo-v3}} \citep{chen21moco-v3}.
    We follow the original parameter settings, where the learning rate is set to 1.5e-4 and the temperature parameter for contrastive loss is 0.2.
    The exponential moving average (EMA) parameter starts from 0.99 and is increased to one during training. 
    The projector consists of 3 MLP layers with an output dimension of 256 and a hidden dimension of 4096. 
    The predictor has 2 MLP layers with a hidden dimension of 4096.
    
    \item \textbf{\textsc{SupMoCo}} \citep{majumder21supmoco}. 
    The learning rate is set to 1.5e-3 and the temperature parameter for contrastive loss is 0.2.
    The EMA parameter starts from 0.99 and is increased to one during training.
    The projector consists of 3 MLP layers with an output dimension of 256 and a hidden dimension of 4096. 
    The predictor has 2 MLP layers with a hidden dimension of 4096.

    \item \textbf{\textsc{BYOL}} \citep{grill20bootstrap}. 
    The learning rate is set to 1.5e-3 and the EMA parameter starts from 0.996 and is increased to one during training.
    The projector consists of 2 MLP layers with an output dimension of 256 and a hidden dimension of 4096. 
    The predictor has 2 MLP layers with a hidden dimension of 4096.

    \item \textbf{\textsc{SupBYOL}}.
    The learning rate is set to 1.5e-3 and
    the EMA parameter starts from 0.996 and is increased to one during training. 
    The loss parameter $\alpha$ is set to 0.5 and all supervised target is obtained by sampling and averaging all positives in the target pool.
    The projector consists of 2 MLP layers with an output dimension of 256 and a hidden dimension of 4096. 
    The predictor has 2 MLP layers with a hidden dimension of 4096.
\end{itemize}

\begin{table*}[h]
\caption{Transfer learning via linear evaluation results on various downstream datasets, where the models trained with the ViT-Small backbone on ImageNet-100.
\textit{CL}, \textit{Sup}, \textit{EMA} stand for the cases when
negative samples are considered, labels are used for pretraining, and the momentum network is adopted, respectively.
\textit{Avg.Rank} represents the average performance ranking across all datasets.
For each dataset, the \textbf{best results} are in bold and the \underline{second-best results} are underlined.
Our proposed methods are marked with $\dagger$.
}
\label{table:vit}
\begin{center}
\resizebox{\textwidth}{!}{%
\begin{tabular}{l|ccc|cccccccccccc}
\toprule
Method & \textit{CL} & \textit{Sup} & \textit{EMA} &  \textit{Avg.Rank} & CIFAR10 & CIFAR100 & DTD & Food & MIT67 & SUN397 & Caltech & CUB200 & Dogs & Flowers & Pets \\
\midrule
\textsc{MoCo-v3} & \cmark &  & \cmark & 4.00 & 84.79 & 64.66 & 60.27 & 59.52 & 56.34 & 45.17 & 75.08 & 37.10 & 44.71 & 85.98 & 64.87 \\
\textsc{SupMoCo} & \cmark & \cmark & \cmark & \underline{2.18} & \underline{89.68} & \textbf{71.07} & 60.90 & 59.84 & 59.18 & 47.45 & \underline{83.13} & \textbf{47.69} & \underline{57.83} & \underline{89.35} & \underline{77.94} \\
\textsc{BYOL} & & & \cmark  & 2.36 & 87.61 & 66.48 & \textbf{65.48} & \textbf{63.36} & \underline{59.48} & \textbf{48.41} & 80.69 & 41.13 & 53.49 & 88.32 & 75.12 \\
\textsc{SupBYOL}$^\dagger$ & & \cmark & \cmark &  \textbf{1.45} & \textbf{90.29} & \underline{71.05} & \underline{62.87} & \underline{61.61} & \textbf{60.07} & \underline{48.36} & \textbf{84.37} & \underline{47.24} & \textbf{62.03} & \textbf{90.31} & \textbf{81.70} \\
\bottomrule
\end{tabular}%
}
\end{center}
\end{table*}

We observe a slightly lower performance in \cref{table:vit} compared to \cref{table:transfer}, where results are presented using ResNet-50 \citep{he2016deep} as the backbone.
This discrepancy is likely due to pretraining with ImageNet-100.
ViT typically requires more data for effective learning compared to ResNet, and the number of data samples in ImageNet-100 may be slightly insufficient.
Nevertheless, supervision in the ANCL framework with the ViT backbone proves effective in enhancing performance. 
Notably, when compared to supervised contrastive learning, proposed method exhibits slightly better performance across all datasets except one. 
This underscores the effectiveness of the supervised ANCL approach, which is applicable to the ViT backbone and remains independent of the underlying architecture.

\section{Pretraining on CIFAR}

We conduct additional experiments on the CIFAR \citep{krizhevsky2009learning} dataset, where the image size was reduced to 32$\times$32. 
The encoder employes a CIFAR variant of ResNet-18 \citep{he2016deep} and is trained for a total of 1000 epochs with a batch size of 256. 
For the ANCL approach, specifically \textsc{SimSiam} and \textsc{SupSiam}, we utilize a 2-layer MLP projector, and Gaussian blurring is excluded from the augmentation. 
For contrastive learning, we select \textsc{SimCLR} \citep{chen20simple} and its supervised version \textsc{SupCon} \citep{khosla20supcon}.
For ANCL, \textsc{SimSiam} \citep{chen21siamese} and \textsc{BYOL} \citep{grill20bootstrap} and their supervised counterparts \textsc{SupSiam} and \textsc{SupBYOL} are chosen as models.
Learning rates are tuned individually for each model: \textsc{SimCLR} (0.7), \textsc{SupCon} (0.6), \textsc{BYOL} (0.6), \textsc{SupBYOL} (0.5), \textsc{SimSiam} (0.7), and \textsc{SupSiam} (0.7).
For supervised ANCL, the target pool size is reduced to 4096, and the loss parameter $\alpha$ is set to 0.5 for \textsc{SupSiam} and 0.8 for \textsc{SupBYOL}.

\begin{table}[h]
\caption{Comparision of CL and ANCL with their self-supervised / supervised versions with ResNet-18 on CIFAR10 and 100. We run all experiments for 1000 epochs. If the pretext and downstream datasets are aligned, the supervised version shows improved performance. In contrast, when there is a mismatch, performance gains are observed only in the ANCL scenario.}
\label{table:cifar_result}
\begin{center}
\begin{tabular}{cccccccc}
\toprule
Pretext & Downstream & \textsc{SimCLR} & \textsc{SupCon} & 
\textsc{SimSiam} & \textsc{SupSiam} & \textsc{BYOL} & \textsc{SupBYOL} \\
\midrule
\multirow{2}{*}{CIFAR10} & CIFAR10 & 89.58 & 95.15  & 93.36 & 94.73 & 91.56 & 94.88 \\
& CIFAR100 & 56.36 & 53.82 & 60.51 & 61.76 & 50.20 & 55.47 \\
\cmidrule(rl){1-8}
\multirow{2}{*}{CIFAR100} & CIFAR10 & 80.36 & 79.99 & 78.81 & 85.19 & 78.35 & 80.09 \\
& CIFAR100 & 64.77 & 74.03 & 70.63 & 75.05 & 65.42 & 74.36 \\

\bottomrule
\end{tabular}
\end{center}
\end{table}

The results in \cref{table:cifar_result} indicate that when the pretext and downstream datasets are the same, the introduction of supervision leads to an increase in linear accuracy. 
Conversely, in cases where they differ, contrastive learning shows a decline or slight increase in performance. 
Asymmetric non-contrastive learning, on the other hand, benefits from labels, resulting in increased accuracy and showcasing the best performance.
Thus, our proposed supervised ANCL proves to be an effective method for obtaining high-quality representations across various datasets.

\end{document}